\documentclass[lettersize,journal]{IEEEtran}
\pdfoutput=1
\usepackage{amsmath,amsfonts}
\usepackage{amssymb} 
\usepackage{algorithm}
\usepackage{array}
\usepackage{booktabs} 
\usepackage{textcomp}
\usepackage{stfloats}
\usepackage{url}
\usepackage{verbatim}
\usepackage{graphicx}
\usepackage{cite}
\usepackage{algpseudocode}
\usepackage{amsthm}
\theoremstyle{definition}

\theoremstyle{plain}
\newtheorem{theorem}{Theorem}
\newtheorem{lemma}{Lemma}
\newtheorem{prop}{Proposition}
\newtheorem{defi}{Definition}

\newtheorem{assu}{Assumption}

\usepackage{subfigure}

\usepackage{mathtools}
\mathtoolsset{showonlyrefs}

\usepackage{xcolor}

\begin{document}

\title{Mobility Accelerates Learning: Convergence Analysis on Hierarchical Federated Learning in Vehicular Networks}



\author{Tan Chen, Jintao Yan, Yuxuan Sun,~\IEEEmembership{Member,~IEEE}, Sheng Zhou,~\IEEEmembership{Member,~IEEE}, \\Deniz G\"und\"uz,~\IEEEmembership{Fellow,~IEEE}, Zhisheng Niu,~\IEEEmembership{Fellow,~IEEE}
\thanks{T. Chen, J. Yan, S. Zhou and Z. Niu are with the Beijing National Research Center for Information Science and Technology, Department of Electronic Engineering, Tsinghua University, Beijing 100084, China (e-mail: \{chent21,yanjt22\}@mails.tsinghua.edu.cn,
\{sheng.zhou,niuzhs\}@tsinghua.edu.cn).}
\thanks{Y. Sun (Corresponding Author) is with the School of Electronic and Information Engineering, Beijing Jiaotong University, Beijing 100044, China (e-mail: yxsun@bjtu.edu.cn).}
\thanks{D. G\"und\"uz is with the Department of Electrical and Electronic Engineering, Imperial College London, SW7 2BT, UK (e-mail: d.gunduz@imperial.ac.uk).}}
\maketitle

\begin{abstract}
Hierarchical federated learning (HFL) enables distributed training of models across multiple devices with the help of several edge servers and a cloud edge server in a privacy-preserving manner. 
In this paper, we consider HFL with highly mobile devices, mainly targeting at vehicular networks. Through convergence analysis, we show that mobility influences the convergence speed by both fusing the edge data and shuffling the edge models. While mobility is usually considered as a challenge from the perspective of communication, we prove that it increases the convergence speed of HFL with edge-level heterogeneous data, since more diverse data can be incorporated. Furthermore, we demonstrate that a higher speed leads to faster convergence, since it accelerates the fusion of data. Simulation results show that mobility increases the model accuracy of HFL by up to 15.1\% when training a convolutional neural network on the CIFAR-10 dataset.

\end{abstract}

\begin{IEEEkeywords}
Hierarchical federated learning, mobility, convergence analysis, heterogeneous data
\end{IEEEkeywords}

\section{Introduction}



The advent of 5G has revolutionized intelligent vehicles, enabling them to generate and share vast amounts of data through vehicle-to-everything (V2X) services\cite{elbir2022federated}. In this context, machine learning (ML) has become an essential tool to efficiently exploit patterns of data while adapting to the dynamics of the mobile environment\cite{sun2022meet,tan2020federated,liang2018toward}. 
\textcolor{black}{Conventional ML solutions rely on offloading the data collected by mobile devices to the cloud server. Such centralized solutions, however, encounter challenges of limited communication resources and privacy concerns in vehicular networks \cite{ye2020federated,posner2021federated,zhou2023toward}. Federated learning (FL), where multiple devices cooperatively train a model by updating the model parameters with their local data and sharing the gradients through a parameter server, emerges as a privacy-preserving alternative, \cite{elbir2022federated,mcmahan2017communication}.
FL has great potentials in various vehicular network scenarios, such as edge caching\cite{yu2020mobility}, resource allocation\cite{samarakoon2019distributed}, and radio-based localization\cite{yin2020fedloc}.}

Nevertheless, FL faces challenges when applied to vehicular networks. Traditionally, the parameter server of FL is located on a remote cloud. However, the communication links between the cloud server and vehicles leads to large communication delays and dropout rate\cite{lim2021dynamic}, resulting in a long training time\cite{sun2020edge}. \textcolor{black}{Training tasks such as edge caching and channel estimation are required to be trained up-to-date based on the changing data generated by the dynamic environment, and the timeliness of these tasks can not be guaranteed by the cloud-based FL.} Federated edge learning (FEEL), where an edge server acts as the parameter server, has a lower communication delay to vehicles\cite{gunduz2020communicate,shi2020communication}. However, the edge server generally covers a small area, and the density of vehicles is low. As a result, the number of participating vehicles is often too small to make the model converge\cite{elbir2022federated}.
To balance the trade-off between training delay and coverage range, 
hierarchical federated learning (HFL) has been proposed\cite{abad2020hierarchical,liu2020client}, which trains a global model at a central cloud server in a federated manner with the help of edge servers. \textcolor{black}{The wide coverage of the cloud server ensures sufficient participating vehicles, while the edge servers are mediators for model aggregation, which help decrease the total training time.}

\IEEEpubidadjcol In this work, we consider HFL in vehicular networks (see Fig. \ref{HFL} for an illustration). \textcolor{black}{The performance of HFL in vehicular networks is mainly affected by two factors.} Firstly, different vehicles have \textcolor{black}{diverse} routes and may collect data with \textcolor{black}{diverse} statistics, e.g., different classes, resulting in data heterogeneity (or, non-i.i.d. data)\cite{ayache2023walk,mcmahan2017communication}. Heterogeneous data causes the local objective function to diverge from the global objective function, thereby degrading the learning performance\cite{shi2020joint,wang2021addressing}. Many research efforts have been devoted to the problem of heterogeneous data in FL and HFL, but few in the context of vehicular networks.
Secondly, unlike traditional HFL where the devices are fixed, the mobility of vehicles dynamically changes the network topology. \textcolor{black}{The impact of mobility in HFL is examined in \cite{feng2022mobility}, but the authors claim that mobility deteriorates the training performance by decreasing the number of vehicles that successfully upload their models.}

We have a different perspective on the impact of mobility on HFL. On the one hand, the movements of vehicles cause frequent cell handovers and channel variations, which makes the uploading of models from the vehicles to the edge servers more difficult. Vehicles may frequently receive the edge model from one edge server, while uploading its local model to another. This phenomenon is called \textit{shuffling the edge models}.
On the other hand, mobility promotes mixing of heterogeneous data, because vehicles can act as `data mules'\cite{shah2003data}, which potentially improve the learning performance.  This is called \textit{fusing the edge data}.

Therefore, in this work, we investigate the effect of mobility on the performance of HFL. The main contributions of our work can be summarized as follows.

\color{black}
\begin{itemize}

\item We analyze the upper bound of the loss function of general HFL tasks \textit{w.r.t.} heterogeneous data and general mobility models, showing how mobility influences the bound. 

\item 
We further analyze the performance of HFL for classification tasks with Markov mobility of vehicles. Three typical initial label distributions are discussed and theoretical results demonstrate that mobility \textit{improves} the performance of HFL with edge-level label imbalance. Specially, when the edge servers form a ring topology, a higher vehicle speed leads to faster convergence.
 
\item Through simulations we show that mobility indeed \textit{enhances} the convergence speed and accuracy of HFL 
with edge-level label imbalance. Furthermore, the potential of mobility to reduce resource consumption is studied.
\end{itemize}
\color{black}

\begin{figure}[t]
\centering
\includegraphics[width=0.45\textwidth]{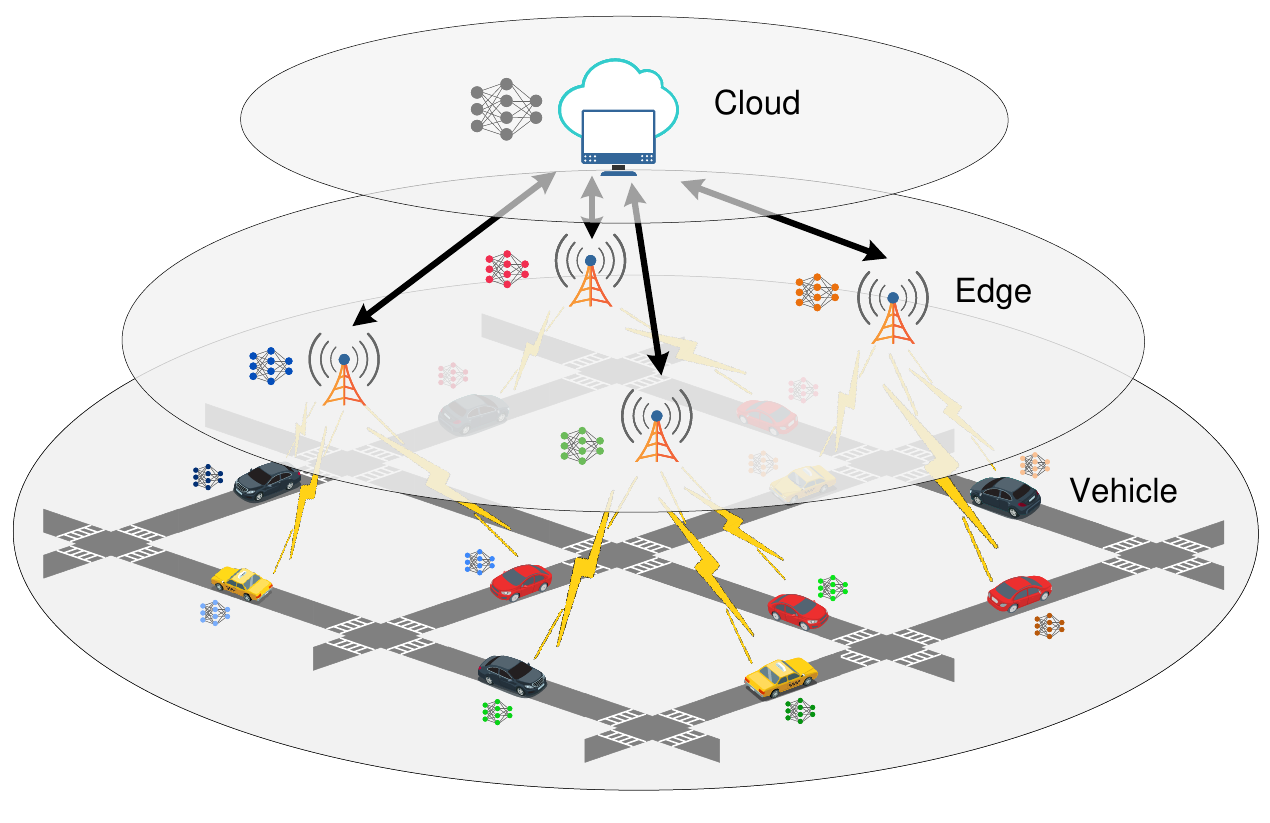}
\caption{Illustration of HFL in vehicular networks.}
\label{HFL}
\end{figure}

The rest of this paper is organized as follows. Section \ref{Sec-2} introduces the related work. Section \ref{Sec-3} describes the system model, characterizes the learning task, and presents the training algorithm. In Section \ref{Sec-4}, convergence analysis of the system is conducted, showing the impact of mobility on HFL with heterogeneous data. Section \ref{Sec-5} presents simulation results and discussions. Finally, Section \ref{Sec-6} concludes the work.
\section{Related Work}
\label{Sec-2}
\begin{figure*}
\centering
\includegraphics[width=0.8\textwidth]{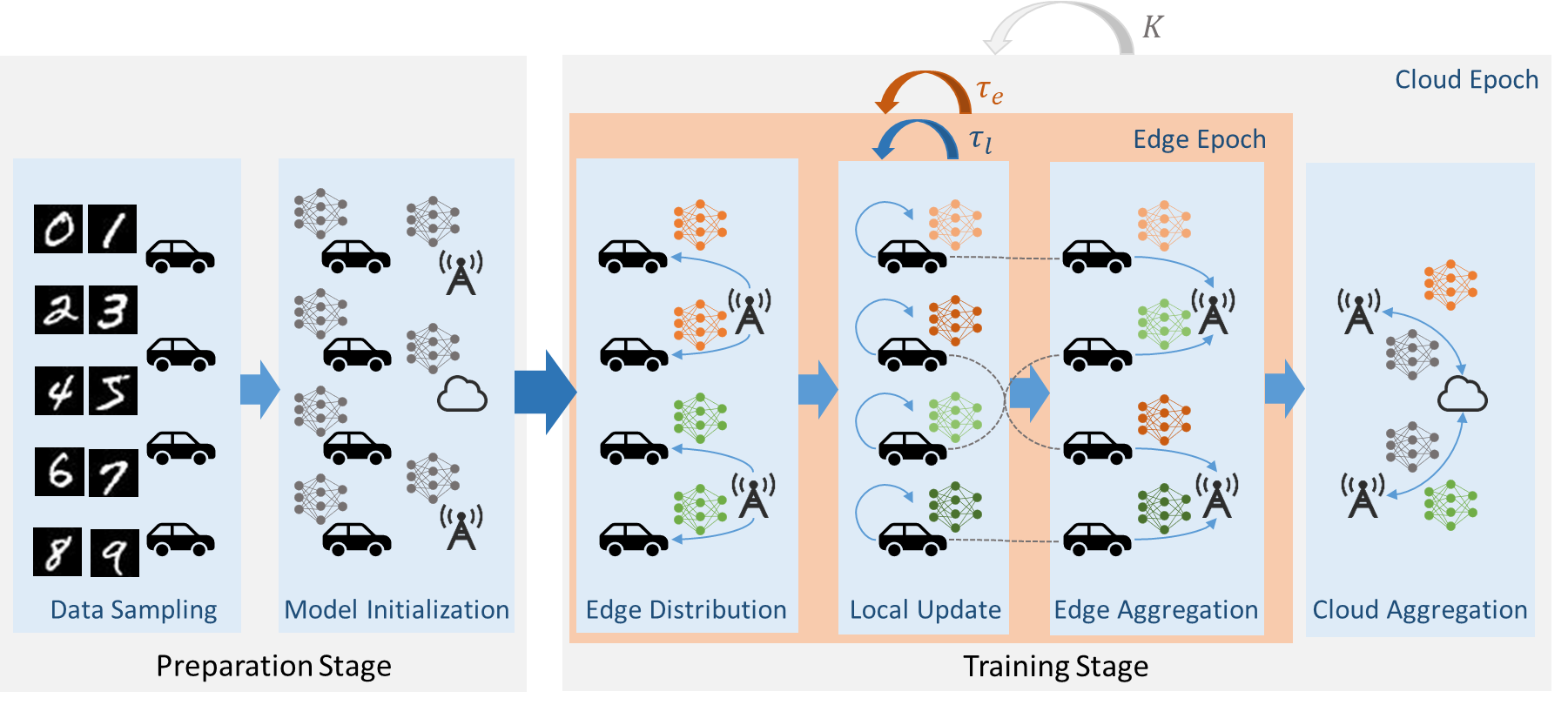}
\caption[center]{The training procedure of HFL.}
\label{Training Procedure}
\end{figure*}

The issue of heterogeneous data has gained attention ever since the inception of FL. In \cite{mcmahan2017communication}, it is highlighted that non-i.i.d. data across devices is a natural occurrence in FL. To date, numerous efforts have been made to address this problem, which can be broadly categorized into three types: data-based, algorithm-based, and framework-based\cite{ma2022state,zhu2021federated}. Data-based methods aim to balance data among different devices by sharing a small portion of global data\cite{zhao2018federated} or generating augmented samples\cite{jeong2018communication}. Algorithm-based methods focus on designing specific aggregation\cite{karimireddy2020scaffold,li2020federated} or optimization strategies\cite{reddi2020adaptive}. Framework-based methods employ new paradigms to model the data heterogeneity of devices, such as multi-task learning\cite{smith2017federated}, personalized learning\cite{arivazhagan2019federated}, and clustering\cite{ghosh2020efficient, wang2021adaptive}. Nonetheless, these studies treat data heterogeneity as a generic phenomenon\cite{zhang2022federated}, and most of them examine the performance of their algorithms by simulations without much analysis on the convergence bound of FL.

The authors of \cite{kairouz2021advances} classify data heterogeneity into four specific types: feature distribution skew, label distribution skew, concept shift between features and labels, and quantity skew. Among them, label distribution skew, also known as \textit{class imbalance}, is more prevalent in realistic scenarios. Consequently, most subsequent researches focus on this aspect.

Simulations have demonstrated that imbalanced data negatively impacts the training performance\cite{wang2021addressing}. A mitigation strategy is then proposed, which measures the class imbalance by the ratio between the number of samples in the majority class and the minority class, and adds the ratio loss to the training loss. 
In \cite{duan2019astraea}, devices are categorized into uniform and biased groups based on their data distributions. Mediators are then introduced to manage the training and balance biased devices according to the Kullback-Leibler (KL) divergence of their data distributions. The authors of \cite{zhang2022federated} adjust the cross-entropy loss based on the occurrence of each class to train a more balanced local model.
The final classifier layer of the neural network, typically a softmax layer, is modified in \cite{duan2019astraea,luo2021no}. A scaling factor is introduced into the softmax function to accommodate missing classes in \cite{duan2019astraea}, while the i.i.d. raw features are employed to calibrate the trained model \cite{luo2021no}. The accuracy decrease in FL with non-i.i.d. data is attributed to weight divergence in \cite{zhao2018federated}, where earth mover's distance (EMD) is introduced as a measure to quantitatively assess weight divergence. Although the class imbalance problem is effectively addressed in these work, vehicular network scenarios are not considered, where the mobility of vehicles impact the training performance.

In the context of vehicular networks, 
to balance the training time and the number of participating vehicles, HFL structure is considered, where vehicles cooperatively train a global model with the help of the cloud and edge servers.
Several work attempt to reduce the degree of heterogeneity of HFL by shaping the data distribution at the edge. The user-edge association problem is often considered, which involves allocating users that fall under the coverage of multiple edge servers. By minimizing the KL divergence among the data distribution of each edge server, a user-edge association method is adopted in \cite{deng2021share}, reducing the total number of communication rounds during training. The authors of \cite{liu2022joint} model the optimal user-edge association problem by minimizing the weight divergence between the global and optimal models, which can be further converted to data distribution imbalance. As a result, the assignment problem is transformed into a distributed optimization problem. Despite the value of these work, 
they do not consider the change in data distribution when vehicles move.

The impact of mobility in HFL is examined in \cite{feng2022mobility}. The convergence speed of HFL in relation to mobility rate is analyzed, showing that mobility decreases the convergence speed of HFL, and an algorithm is then proposed to mitigate the impact of mobility by adopting cosine similarity as the aggregating weights. \textcolor{black}{However, the authors ignore vehicles that cross the coverage of edge servers (crossing vehicles) when edge servers aggregate models of vehicles. In contrast, we think the data on the crossing vehicles is important as they help mix diverse data, so these vehicles should also be considered.}

Theoretical analysis of FL with heterogeneous data has been conducted by several work. The authors of \cite{wang2019adaptive, li2019convergence} measure non-i.i.d. by the upper bound of gradient divergence and the difference between the optimal values of the local and global loss functions, respectively. The corresponding convergence bounds with respect to the non-i.i.d. metrics are then derived. Subsequent work \cite{liu2020client,feng2022mobility,wang2022demystifying,liu2022hierarchical,cai2022high,wu2023hiflash} further analyze the convergence bound for HFL with heterogeneous data. These work does not consider the scenario of vehicular networks, so the impact of mobility is not examined.

The authors of \cite{sun2022meet} suggest that mobility has two-fold impacts on data-heterogeneous FL. On the one hand, it brings uncertainties to the network topology, \textcolor{black}{which may decrease the number of vehicles that successfully upload the model or deteriorate the quality of the model}. On the other hand, it facilitates data fusion which potentially accelerates the convergence speed of training. Simulations show that mobility increases training accuracy of the edge-based FL. Following this work, we further exploit the impact of mobility in the scenario of data-heterogeneous HFL. 
\section{HFL in Vehicular Networks}
\label{Sec-3}


We consider HFL in vehicular networks. A central cloud server \textcolor{black}{equipped with a macro base station controls several edge servers, which can represent base stations or roadside units.} Each edge server is static, and covers
a limited area of streets. The vehicles
move on the streets and cross the
coverage of an edge server occasionally. We assume that the cloud server has a large enough coverage, so all the vehicles are always
within its coverage. Therefore, the cloud and edge servers, and the vehicles together form a closed system without input and output flows. 
We denote the central cloud as C, and assume that there are $N$ edge servers and $M$ mobile vehicles. In practice, $M$ is much larger than $N$. The vehicles in the system cooperatively train a model with the help of the cloud and edge servers.

The HFL task attempts to find the connection between inputs $\boldsymbol{x}_i$ and labels $y_i$ in the global dataset $\mathcal{D}=\{\boldsymbol{x}_i,y_i \}_{i=1}^{|\mathcal{D}|}$. Let $\mathcal{D}_m$ denote the dataset of the $m$-th vehicle; we have $\mathcal{D}=\cup_{m=1}^M \mathcal{D}_m$. 

For a sample $\{\boldsymbol{x}_i,y_i \}$, let $g_i \left(\boldsymbol{w}\right)$ be the sample loss function for models $\boldsymbol{w}$. The loss function of the $m$-th vehicle is given by $f_m\left(\boldsymbol{w}\right)=\frac1{|\mathcal{D}_m|} \sum_{i\in \mathcal{D}_m}g_i \left(\boldsymbol{w}\right)$, and the global loss function is then determined by an average of the sample loss functions, $F\left(\boldsymbol{w}\right)=\frac1{|\mathcal{D}|} \sum_{i\in \mathcal{D}}g_i \left(\boldsymbol{w}\right)= \sum_{m=1}^M \frac{|\mathcal{D}_m|}{|\mathcal{D}|} f_m\left(\boldsymbol{w}\right) $. The training objective  of HFL is $\min_{\boldsymbol{w}} F\left(\boldsymbol{w}\right)$.

\subsection{Mobility-aware Hierarchical Federated Averaging}
\begin{table}[t]\large
  \caption{Summary of Main Notations}
  \label{tab1}
  \renewcommand\arraystretch{1.5}
  \resizebox{\linewidth}{!}{
\begin{tabular}{cp{8.5cm}}
   \toprule
   $\boldsymbol{ \mathrm{Notation}}$& $\boldsymbol{ \mathrm{Definition\centering}}$\\
   \midrule
   $\mathcal{D}$; $\mathcal{D}_{m}$& global dataset; dataset of the $m$-th vehicle\\
    $F\left(\boldsymbol{w}\right)$; $f_m\left(\boldsymbol{w}\right)$ & global loss function; loss function of the $m$-th vehicle \\
    $\tau$ & number of total local updates\\
    $\boldsymbol{w}^{\left(\tau\right)}$; $\boldsymbol{w}_{e,n}^{\left(\tau\right)}$; $\boldsymbol{w}_{m}^{\left(\tau\right)}$ & model at the cloud server, $n$-th edge server and $m$-th vehicle\\
    $\boldsymbol{\xi}_{m}^{\left(\tau\right)}$ & the batch sampled by the $m$-th vehicle\\
    $g\left(\boldsymbol{w}_{m}^{\left(\tau\right)}\right)$ & gradients calculated over the batch $\boldsymbol{\xi}_{m}^{\left(\tau\right)}$\\
    $\tau_l$; $\tau_e$; $K$  & local period; edge period; cloud period\\
    $\theta^{(\tau)}_n$; $\alpha^{(\tau)}_{m,n}$ & aggregation weights of the $n$-th edge server and $m$-th vehicle\\
   \bottomrule
\end{tabular}
}
\end{table}


We follow the HFL framework in \cite{liu2020client,abad2020hierarchical} while further incorporating the mobility of vehicles. We call the stochastic gradient descent (SGD) process on a batch of samples at each vehicle as a local update. We maintain an iteration enumerator $\tau$, which records the number of local updates each vehicle has carried out in total. A synchronized system is considered, so $\tau$ is tracked by all the vehicles. We denote the models of the $m$-th vehicle, the $n$-th edge server, and the cloud server at iteration $\tau$ as $\boldsymbol{w}_{m}^{\left(\tau\right)}, \boldsymbol{w}_{e,n}^{\left(\tau\right)}, \boldsymbol{w}^{\left(\tau\right)}$, respectively.

The training procedure is illustrated in Fig. \ref{Training Procedure}. We assume that vehicles sample data before training in the preparation stage, and the dataset carried by each vehicle does not vary during training. After sampling, all the cloud and edge servers initialize with a global model $\boldsymbol{w}^{(0)}$. The training stage is consist of the following four steps:

    \textbf{1) Edge distribution}: Let each edge server maintain a covering vehicle set $\mathcal{E}_n^{\left(\tau\right)}$, and assume that each vehicle is associated with the nearest edge server only. 
    Each edge server distributes the edge model to all the vehicles within its coverage, i.e., $\boldsymbol{w}_{m}^{(\tau)}=\boldsymbol{w}_{e,n}^{(\tau)}, m\in \mathcal{E}_n^{\left(\tau\right)}$.

     \textbf{2) Local update}: 
     For each local update, the $m$-th vehicle performs SGD using its local data by ${\Tilde{\boldsymbol{w}}}_{m}^{\left(\tau\right)}= \boldsymbol{w}_{m}^{({\tau}-1)}-\eta \nabla f_{m} \big(\boldsymbol{w}_{m}^{\left(\tau-1\right)},\boldsymbol{\xi}_{m}^{\left(\tau-1\right)}\big)$, where $\eta$ is the learning rate, and $\nabla f_{m} \big(\boldsymbol{w}_{m}^{\left(\tau\right)},\boldsymbol{\xi}_{m}^{\left(\tau\right)}\big)$ is the stochastic gradient of the loss function with data batch $\boldsymbol{\xi}_{m}^{\left(\tau\right)}$ sampled from $\mathcal{D}_m$. For notation simplicity, we define $g\big(\boldsymbol{w}_{m}^{\left(\tau\right)}\big) \triangleq \eta \nabla f_{m} \big(\boldsymbol{w}_{m}^{\left(\tau\right)},\boldsymbol{\xi}_{m}^{\left(\tau\right)}\big)$. Local update is repeated $\tau_l$ times until the next edge aggregation starts.

    \textbf{3) Edge aggregation}: Each edge server first updates the covering vehicle set $\mathcal{E}_n^{\left(\tau\right)}$. Then it collects and aggregates models from all the vehicles within its coverage at that time. The aggregated model at the $n$-th edge server is denoted by $\Tilde{\boldsymbol{w}}_{e,n}^{\left(\tau\right)}=\sum_{m\in \mathcal{E}_n^{\left(\tau\right)}}\alpha_{m,n}^{\left(\tau\right)} \Tilde{\boldsymbol{w}}_{m}^{\left(\tau\right)}$, where $\alpha^{(\tau)}_{m,n}\!\triangleq\!\frac{|\mathcal{D}_{m}|}{\sum_{m'\in \mathcal{E}_n^{\left(\tau\right)}}|\mathcal{D}_{m'}|}$ denotes the aggregation weight of the $m$-th vehicle.

    The edge distribution, the local update, and the edge aggregation form one edge epoch, which is repeated $\tau_e$ times until the next cloud aggregation starts.
    
    \textbf{4) Cloud aggregation}: The cloud server receives, aggregates, and distributes edge models. The aggregation process is expressed by $\boldsymbol{w}_{e,n}^{\left(\tau\right)} = \boldsymbol{w}^{\left(\tau\right)} = \sum_{n=1}^N \theta^{(\tau)}_n \Tilde{\boldsymbol{w}}_{e,n}^{\left(\tau\right)} $, where $ \theta^{(\tau)}_{n}\!\triangleq\!\frac{\sum_{m\in \mathcal{E}_n^{\left(\tau\right)}}|\mathcal{D}_{m}|}{\sum_{m=1}^M|\mathcal{D}_{m}|}$ denotes the weight of the $n$-th edge server. \vspace{0.5mm}

    All the four steps above form one cloud epoch, which is repeated $K$ times. We call $\tau_l, \tau_e, K$ the local period, the edge period and the cloud period, respectively. 

    Following the steps above, the evolution of the local model at the $m$-th vehicle is denoted by
\begin{equation}
        \boldsymbol{w}_{m}^{\left(\tau\right)}=\begin{cases} \Tilde{\boldsymbol{w}}_{m}^{\left(\tau\right)},\quad \tau_l\nmid\tau \\     \Tilde{\boldsymbol{w}}_{e,n}^{\left(\tau\right)},\quad \tau_l\mid\tau, \tau_l\tau_e\nmid\tau \\         \boldsymbol{w}^{\left(\tau\right)},\quad \tau_l\tau_e\mid\tau \end{cases}\hspace{-9pt}\!,\! 
\end{equation}

     \noindent Here $a\mid b$ means $b$ is 
    divisible by $a$, while $a\nmid b$ means $b$ is not divisible by $a$.
    
The main notations are listed in Table \ref{tab1}, and the details of the training procedure are described in Algorithm \ref{alg1}.

\subsection{Data Distribution}
The performance of HFL is affected by the data heterogeneity during training, which is mainly related to the initial data distribution and the mobility of vehicles. We consider three typical sampling cases of initial data distributions:

(1) \emph{i.i.d.} The data patterns are homogeneous for all vehicles, and each vehicle owns a dataset of the same distribution.

\color{black}
(2) \emph{edge non-i.i.d.} The training tasks have spatial dependence, such as trajectory forecasting and beam selection, where the data patterns depend on the location of vehicles. In this case, the data distributions are non-i.i.d. across edges, but i.i.d. among the vehicles within the same edge server.

(3) \emph{local non-i.i.d.} The data patterns are heterogeneous for different vehicles. In this case, vehicles own datasets of different distributions. If the number of vehicles is large enough, the dataset of each edge server (denoted by the union of the datasets of all the vehicles within its coverage) is assumed to be i.i.d. (edge i.i.d.), since the data of vehicles is sufficiently mixed up; otherwise local non-i.i.d. leads to edge non-i.i.d..
\color{black}

\begin{algorithm}[t]
\caption{Mobility-aware Hierarchical Federated Averaging (Mob-HierFAVG)}
\label{alg1}
\begin{algorithmic}[1] 
\Require $\eta$
    \Procedure{MobilityAwareHierarchicalFederatedAveraging}{}
    \State Initialize the cloud model with parameter $\boldsymbol{w}^{(0)}$
    \For{$k = 0, \ldots, K$}
        
        \For{$t_e = 0, \ldots, \tau_e$}
            \State EdgeDistribute$\left(\{\boldsymbol{w}_{e,n}^{(\tau)}\}_{n=1}^N,\{\boldsymbol{w}_{m}^{(\tau)}\}_{m=1}^M\right)$
            \For{$t_l = 1, \ldots, \tau_l$}
                \State $\tau\gets\tau+1$
                \State LocalUpdate$\left(\{\boldsymbol{w}_{m}^{(\tau)}\}_{m=1}^M\right)$
            \EndFor
            \ForAll{edge server $n = 1, \ldots, N$ in parallel}
                \State
                $\mathcal{E}_n^{\left(\tau\right)} \gets \mathcal{E}_n^{\left(\tau-1\right)}$
                \State ${\boldsymbol{w}}_{e,n}^{\left(\tau\right)} \gets \sum_{m\in \mathcal{E}_n^{\left(\tau\right)}}\alpha_{m,n} {\boldsymbol{w}}_{m}^{\left(\tau\right)}$   
            \EndFor\Comment{EdgeAggregation}
            
        \EndFor
        \State$\boldsymbol{w}^{\left(\tau\right)} \gets \sum_n \theta_n {\boldsymbol{w}}_{e,n}^{\left(\tau\right)} $
        \State$\{\boldsymbol{w}_{e,n}^{(\tau)}\}^N_{n=1}\gets\boldsymbol{w}^{\left(\tau\right)}$           \Comment{CloudAggregation}
    \EndFor
    \EndProcedure
    \Function{EdgeDistribute}{$\{\boldsymbol{w}_{e,n}^{(\tau)}\}_{n=1}^N,\{\boldsymbol{w}_{m}^{(\tau)}\}_{m=1}^M$}
        \ForAll{edge server $n = 1, \ldots, N$ in parallel}
            \ForAll{vehicle $m\in\mathcal{E}_n^{(\tau)}$ in parallel}
                \State $\boldsymbol{w}_{m}^{(\tau)}\gets\boldsymbol{w}_{e,n}^{(\tau)}$
            \EndFor
        \EndFor
    \EndFunction
    \Function{LocalUpdate}{$\{\boldsymbol{w}_{m}^{(\tau)}\}_{m=1}^M$}
        \ForAll{client $m = 1, \ldots, M$ in parallel}
            \State ${\boldsymbol{w}}_{m}^{\left(\tau\right)} \gets \boldsymbol{w}_{m}^{({\tau}-1)}-g\left(\boldsymbol{w}_{m}^{\left(\tau-1\right)}\right)$
        \EndFor
    \EndFunction
\end{algorithmic}
\end{algorithm}

\subsection{Mobility Model}

Since the edge server that each vehicle belongs to at the time of edge aggregation is needed, the mobility of vehicles is modeled by a discrete Markov chain. The observation interval of the Markov chain is equal to the interval between two consecutive edge aggregations. The states of the Markov chain represent the relationship between the edge server and the vehicle, with a state space size $N$. 
Denote the transition probability matrix by $\boldsymbol{P}\in\mathcal{R}^{N\times N}$, where the $i$-th row and $j$-th column element $\boldsymbol{P}_{ij}$ represents the probability of a vehicle that stays within the $i$-th edge server at the last epoch and moves to the coverage of the $j$-th edge server at the current epoch.

\section{Convergence Analysis}
\label{Sec-4}
\color{black}
In this section, the upper bound on the loss function of the Mob-HierFAVG algorithm for general tasks and classification tasks are derived. The impact of mobility is then discussed based on the bounds.
\color{black}

\subsection{Convergence of HFL in static and mobile scenarios}
We first quantify data heterogeneity by the gradient difference.
Define $\delta_m$ and $\Delta_n^{(\tau)}$ as
\begin{align}
    {\left\lVert {\nabla f_m\left(\boldsymbol{w}\right)-\nabla F\left(\boldsymbol{w}\right)} \right\rVert}
    &\le \delta_m, \\
    {\left\lVert {\nabla F_{e,n}^{(\tau)}\left(\boldsymbol{w}\right)-\nabla F\left(\boldsymbol{w}\right)} \right\rVert }
    &\le \Delta_n^{(\tau)},
\end{align}
where $F_{e,n}^{(\tau)}\left(\boldsymbol{w}\right)\triangleq\sum_{m\in \mathcal{E}_n^{(\tau)}}\alpha_{m,n}^{(\tau)}f_{m}\left(\boldsymbol{w}\right).$ \textcolor{black}{Here, $\delta_m$ and $\Delta_n^{(\tau)}$ represent the data distribution difference between the dataset of the $m$-th vehicle and the cloud server, and between the $n$-th edge server and the cloud server at the $\tau$-th local iteration, respectively.}

Furthermore, the edge-local gradient difference is represented by
\begin{equation}
    \delta_n^{(\tau)}\triangleq\sum_{m\in \mathcal{E}_n^{(\tau)}} \alpha_{m,n}^{(\tau)}\delta_m,
\end{equation}
and the cloud-local and cloud-edge gradient differences are 
\begin{align}
    \delta\triangleq\sum_{m=1}^M \alpha_m\delta_m,\quad
    \Delta^{(\tau)}\triangleq\sum_{n=1}^N \theta_n^{(\tau)}\Delta_n^{(\tau)},
\end{align}
respectively, where $\alpha_m\triangleq\frac{|\mathcal{D}_m|}{|\mathcal{D}|}$.

The definition indicates that the cloud-local gradient difference is not time-varying, while the cloud-edge and edge-local gradient differences are. \textcolor{black}{This is because the whole datasets of vehicles remain the same during training, while the dataset of each edge server varies over time due to the mobility of vehicles.} Therefore, $\delta_m$ and $\Delta_n^{(\tau)}$ represent the \textbf{data heterogeneity} impacted by the \textbf{mobility} of vehicles.

Next, a general convergence bound for HFL is introduced based on Lemma 2 in \cite{wang2019adaptive}.

Define the following virtual models:
\begin{enumerate}
    \item $\boldsymbol{u}^{(\tau)}$: The virtual cloud model. It records the weighted sum of local models at each iteration by:
    \begin{equation}
        \boldsymbol{u}^{(\tau)}\triangleq\sum_{m}\alpha_{m} \boldsymbol{w}_{m}^{(\tau)}.
    \end{equation}\vspace{-8pt}
    
    \item $\boldsymbol{v}^{(\tau)}$: The virtual centralized model. It monitors the process of centralized learning and synchronizes with the virtual cloud model periodically. It evolves as
    \begin{equation}
        \boldsymbol{v}^{(\tau)}\triangleq\begin{cases} 
        \Tilde{\boldsymbol{v}}^{(\tau)}=\boldsymbol{v}^{(\tau-1)}-\eta\nabla F\left(\boldsymbol{v}^{(\tau-1)}\right),\quad \tau_l\tau_e\nmid\tau \\ 
        \boldsymbol{u}^{(\tau)},\quad \tau_l\tau_e\mid\tau \end{cases}\hspace{-9pt}.
    \end{equation}\vspace{-10pt}
\end{enumerate}

\color{black}
The upper bound on the central-federate (CF) difference at the $k$-th cloud epoch is denoted by $U_k$, where
    \begin{align}
    &{\left\lVert{\boldsymbol{u}^{\left(k\tau_l\tau_e\right)}-\Tilde{\boldsymbol{v}}^{\left(k\tau_l\tau_e\right)}}\right\rVert}\le U_k.
    \end{align}
$U_k$ represents
the difference of the convergence speed between HFL and centralized
learning. Also, denote the optimal model parameters by $\boldsymbol{w}^*$.

\begin{assu}\label{assu1} We assume the following for all $m$:

\begin{enumerate}
\item $f_m\left(\boldsymbol{w}\right)$ is convex;
\item $f_m\left(\boldsymbol{w}\right)$ is $\rho$-Lipschitz, i.e., 

$\left\lVert {f_m\left(\boldsymbol{w}\right)-f_m\left(\boldsymbol{w}' \right)} \right\rVert\le \rho \left\lVert {\boldsymbol{w}-\boldsymbol{w}'}  \right\rVert$ for any $\boldsymbol{w},\boldsymbol{w}'$;

\item $f_m\left(\boldsymbol{w}\right)$ is $\beta$-smooth, i.e., 

$\left\lVert {\nabla f_m\left(\boldsymbol{w}\right)-\nabla f_m\left(\boldsymbol{w}'\right)} \right\rVert\le \beta \left\lVert {\boldsymbol{w}-\boldsymbol{w}'}  \right\rVert$ for any $\boldsymbol{w},\boldsymbol{w}'$.
\end{enumerate}

\end{assu}

\begin{prop}\label{prop1}
    For any $m$, if assumption \ref{assu1} holds, and for some $\epsilon\ge 0$, we have
    \begin{enumerate}
        \item[(1)] $\eta\le\frac1{\beta}$,
        \item[(2)] $\eta\varphi\!-\frac{\rho U_k}{\tau_l\tau_e\epsilon^2}\!>\!0$ for all $k$, 
        where $\varphi\!\triangleq\!\min\limits_k\frac{1-\frac{\beta\eta}{2}}{\left\lVert{\Tilde{\boldsymbol{v}}^{\left(\!\left(k\!-\!1\right)\tau_l\tau_e\!\right)}-\boldsymbol{w}^*}\!\right\rVert^2}$
        \item[(3)]
        $F\left(\Tilde{\boldsymbol{v}}^{\left(k\tau_l\tau_e\right)}\right)-F\left(\boldsymbol{w}^*\right)\ge\epsilon$,\label{cond}
        \item[(4)] $F\left(\boldsymbol{w}^{\left(k\tau_l\tau_e\right)}\right)\ge\epsilon$ for all $k$,
    \end{enumerate}
 after $T\triangleq K\tau_l\tau_e $ local updates, the loss function of HFL is bounded by
    \begin{equation}
        F\left(\boldsymbol{w}^{\left(T\right)}\right)-F\left(\boldsymbol{w}^*\right)\le\frac1{T\eta\varphi-\frac{\rho}{\epsilon^2}\sum\limits_{k=1}^KU_k}.
    \end{equation}
\end{prop}
\begin{proof}
    See Appendix \ref{app1}.
\end{proof}

    We then focus on bounding the CF difference in the static scenario and mobile scenario, respectively.
\color{black}

 \subsubsection{Static Scenario}
    In this scenario, we assume vehicles move in a small range and thus do not cross the coverage of edge servers during training. Then the edge gradient difference satisfies $\Delta^{(\tau)}=\Delta$.

\begin{theorem}\label{theo1}
 In the static scenario, the CF difference has an upper bound
\begin{align}\label{staticcfbound}
        U_{k, \mathrm{nomob}}=&r(\tau_l\tau_e, \eta, \delta)
        \!-\!\eta\left(\delta\!-\!\Delta\right)\Big[
        \frac12\tau_e\left(\tau_e\!-\!1\right)\tau_l\!+\!H(\tau_l,\tau_e)\Big],
\end{align}
where $r(\tau, \eta, \delta)=
        \frac{\delta}{\beta}\big[\left(1\!+\!\eta\beta\right)^{\tau}\!-\!1\big]\!-\!\tau\eta\delta$, and $H(\cdot)$ is a non-polynomial function of $\tau_l$ and $\tau_e$.
\end{theorem}

\begin{proof}
    See Appendix \ref{app2}.
\end{proof}

\color{black}
    Theorem \ref{theo1} shows that the upper bound on the CF difference is irrelevant to the cloud epoch $k$, so we further set $U=U_{k, \mathrm{nomob}}$.

Here $r(\tau_l\tau_e, \eta, \delta)$ is equals to the CF difference of FL with aggregation period $\tau_l\tau_e$, learning rate $\eta$, and gradient difference $\delta$. The right hand side of Eq. \eqref{staticcfbound} can be recognized as the improvement of HFL over FL by introducing edge servers as the intermediate layer. 
It is shown that the improvement of HFL is determined by both the data heterogeneity $\delta$, $\Delta$ and aggregation periods $\tau_l$, $\tau_e$. 

Besides, $\tau_l$ and $\tau_e$ are not symmetric. If $\tau_l\tau_e$ is set to a fixed value, then the choice of $\tau_l$ and $\tau_e$ still influences the convergence speed. This is aligned with the conclusions in \cite{feng2022mobility}.
\color{black}

\subsubsection{Mobility Scenario}
We now consider that vehicles have high mobility and may cross the coverage of edge servers.
\begin{theorem}\label{theo2}
 In the mobility scenario, the CF difference has an upper bound
\begin{align}\label{mobcfbound}
 U_{k, \mathrm{mob}}\!=\!
r(\tau_l\tau_e, \eta, \delta)\!-\!\eta\tau_l\Big[\frac12\tau_e\left(\tau_e\!-\!1\right)\delta\!-\!\sum\limits_{j=1}^{\tau_e-1}j\Delta^{[(\!k-1\!)\tau_e+j]}\Big],
\end{align}
where $\Delta^{[j]}=\Delta^{(j\tau_l)}$.
\end{theorem}
\begin{proof}
    See Appendix \ref{app3}.
\end{proof}
\textcolor{black}{Compared with Eq. \eqref{staticcfbound}, $\Delta$ is substituted by $\Delta^{(\tau)}$ in Eq. \eqref{mobcfbound}, because the dataset of the edge server is varying with mobility of vehicles. 
}

To compare the convergence bounds of static and mobility scenarios, we define the mobility factor and derive the form of this factor in some typical cases.

\begin{defi} Define the mobility factor $\gamma_k$ as
    \begin{align}\label{gamma}
        \gamma_k=&\frac{U_{k, \mathrm{nomob}}-U_{k, \mathrm{mob}}}{\tau_l\tau_e}\\
        =&\frac{\eta}{\tau_e}\left[\sum\limits_{j=1}^{\tau_e-1}j\left(\Delta^{[0]}\!-\!\Delta^{[(k\!-\!1)\tau_e+j]}\right)\!-\!\left(\delta\!-\!\Delta^{[0]}\right)\!\frac{H(\tau_l,\tau_e)}{\tau_l}\right]\!,
    \end{align}
which characterizes the effect of mobility on the convergence speed at the $k$-th cloud epoch normalized by local updates. 
\end{defi}

The convergence bound of HFL with mobility is then expressed by
    \begin{align}
        F\left(\boldsymbol{w}_{\text{mob}}^{\left(T\right)}\right)-F\left(\boldsymbol{w}^*\right)
        &\le \frac1{T\left(\eta\varphi-\frac1K\sum\limits_{k=1}^{K}\frac{\rho U_{k, \mathrm{mob}}}{\tau_l\tau_e\epsilon^2}\right)}\\
        &=
        \frac1{T\left[\eta\varphi+\frac{\rho}{\epsilon^2}
        \left(\frac1K\sum\limits_{k=1}^{K}\gamma_k-\frac{U}{T}\right)\right]}.
    \end{align}
Therefore, the effect of mobility with fixed $\eta$, $T$ and $U$ is determined by $\gamma\triangleq\frac1K \sum\limits_{k=1}^{K}\gamma_k$. If $\gamma>0$, mobility accelerates the convergence speed of HFL; if $\gamma=0$, mobility has no apparent influence on the convergence speed; if $\gamma<0$, mobility degrades the performance of HFL.

We also notice that $\gamma_k$ is composed of two parts, with the core elements $\Delta^{[0]}-\Delta^{[(k\!-\!1)\tau_e+j]}$ and $(\delta-\Delta^{[0]})H(\tau_l,\tau_e)$, respectively. Here the first element expresses the evolution of the cloud-edge gradient difference, which is caused by the fusion of data on vehicles as they move across the coverage of edge servers. The second element emerges because some vehicles receive models from one edge server and upload their model updates to another edge server, which decreases the quality of models. Therefore, mobility influences the convergence bound of data-heterogeneous HFL in two ways, \textbf{fusing the edge data} and \textbf{shuffling the edge models}.

\subsection{Convergence of HFL for classification tasks}
To further analyze whether the mobility factor $\gamma$ is positive or negative, the magnitude relationship among $\delta$, $\Delta^{[0]}$ and $\Delta^{[(k\!-\!1)\tau_e+j]}$ should be figured out.

\subsubsection{Mobility Factor for Classification Problem}
$ $

Since $\delta$, $\Delta^{[0]}$ and $\Delta^{[(k\!-\!1)\tau_e+j]}$ are all related to the loss function, we further consider a classification task with $C$ classes, which adopts the cross-entropy loss. Denote $p\left\{y\!=\!c\right\}$ and $p_{n}^{[j]}\left\{y\!=\!c\right\}$ by the proportion of samples with label $c$ in the global dataset and the dataset of the $n$-th edge server at the $j\tau_l$-th iteration, respectively. Then the probability vectors is expressed by $\boldsymbol{p}=[p\left\{y\!=\!1\right\},...,p\left\{y\!=\!c\right\}]$, $\boldsymbol{p}_n^{[j]}=[p_n^{[j]}\left\{y\!=\!1\right\},...,p_n^{[j]}\left\{y\!=\!c\right\}]$.

\begin{assu}\label{assu2}
    The global and edge loss functions for the classification task are denoted by
\begin{align}
F\left(\boldsymbol{w}\right)&=\sum\limits_{c=1}^{C}p\left\{y\!=\!c\right\}\mathbb{E}_{\boldsymbol{x}|y\!=\!c}[\log g_i\left(\boldsymbol{w}\right)], \\
    F_{e,n}^{(j\tau_l)}\left(\boldsymbol{w}\right)&=\sum\limits_{c=1}^{C}p_{n}^{[j]}\left\{y\!=\!c\right\}\mathbb{E}_{\boldsymbol{x}|y\!=\!c}[\log g_i\left(\boldsymbol{w}\right)],
\end{align}

\end{assu}

Assumption \ref{assu2} claims that the gradient difference can be expressed by the weighted sum of expectation of loss function over each class. In other words, the loss function, as well as its gradients, are determined by the class distribution. 

Based on this assumption, the relationship between $\delta$ and $\Delta^{[0]}$ is derived as follows.
\begin{prop}\label{prop2}
    For classification tasks with the i.i.d. and edge non-i.i.d. initial distributions,  $\delta-\Delta^{[0]}=0$.
\end{prop}
\begin{proof}
    For the i.i.d. case, the local and edge data
    distribution is identical to the global data distribution. According to Assumption \ref{assu2}, $\delta=\Delta^{[0]}=0$.
    Similarly, for the edge non-i.i.d. case, the local data
    distribution is identical to the edge data distribution, so we obtain $\delta=\Delta^{[0]}$.
\end{proof}

Proposition \ref{prop2} points out that the convergence bounds of the i.i.d. and edge i.i.d. cases are not impacted by shuffling the edge model. 

To compare $\Delta^{[0]}$ with $\Delta^{[(k\!-\!1)\tau_e+j]}$, we consider deriving an upper bound of $\Delta^{[j]}$. Firstly, we have
\begin{align}
    &\Delta^{[j]}=\left\lVert {\nabla F_{e,n}^{(j\tau_l)}\left(\boldsymbol{w}\right)-\nabla F\left(\boldsymbol{w}\right)} \right\rVert \\ =&\left\lVert\sum\limits_{c=1}^{C}\left(p\left\{y\!=\!c\right\}-p_{n}^{[j]}\left\{y\!=\!c\right\}\right)\nabla_w\mathbb{E}_{\boldsymbol{x}|y\!=\!c}[\log g_i\left(\boldsymbol{w}\right)]\right\rVert\\
    \le& \sum\limits_{c=1}^{C}\left\lVert p\left\{y\!=\!c\right\}-p_{n}^{[j]}\left\{y\!=\!c\right\}\right\rVert\left\lVert\nabla_w\mathbb{E}_{\boldsymbol{x}|y\!=\!c}[\log g_i\left(\boldsymbol{w}\right)]\right\rVert\\
    \le& \lVert\boldsymbol{p}-\boldsymbol{p}_{n}^{[j]}\rVert_1G\label{Delta2},
\end{align}
where $
 G=\max\limits_{c,\boldsymbol{w}}\{|\nabla_{\boldsymbol{w}}\mathbb{E}_{\boldsymbol{x}|y=c}[\log g_i\left(\boldsymbol{w}\right)]|\}
$, and $\lVert\boldsymbol{p}-\boldsymbol{p}_{n}^{[j]}\rVert_1=\sum\limits_{c=1}^{C}\left\lVert p\left\{y\!=\!c\right\}-p_{n}^{[j]}\left\{y\!=\!c\right\}\right\rVert$ denotes the probability difference. Therefore, for a classification problem with cross-entropy loss, the \textcolor{black}{variation} of edge gradient difference is bounded by the \textcolor{black}{variation} of probability difference $\lVert\boldsymbol{p}-\boldsymbol{p}^{[j]}_{n}\rVert_1$.




Next, the mobility model is introduced to further decompose the probability difference. Assume the mobility of vehicles follows a discrete Markov model as stated in Section \ref{Sec-2}. Then the transition probability of vehicles from edge server $n$ to edge server $n'$ during the last edge epoch is denoted by
\begin{align}
    \boldsymbol{Q}_{n,n'}\triangleq&P\{\text {travel from edge server } n \text { to edge server } n' \\ &\text { during the last edge epoch}\}\\
    =&P\left(m\in\mathcal{E}_{n'}^{\left(\left(j+1\right)\tau_l\right)}|m\in\mathcal{E}_n^{\left(j\tau_l\right)} \right).
\end{align}

\textcolor{black}{To simplify the model, we assume that the number of vehicles within the coverage of the $n$-th edge server does not change sharply in two adjacent edge aggregations. That is to say, $\theta_n^{(j\tau_l)}\approx\theta_n^{(j+1)\tau_l}$ for all $n$.}

As the training goes on, the label distributions of edge servers change accordingly by
\begin{equation}
    \boldsymbol{p}_{n}^{[j+1]}=
    \frac{\sum\limits_{n'}\theta_{n'}^{(j\tau_l)}\boldsymbol{Q}_{n',n}\boldsymbol{p}_{n'}^{[j]}}{\theta_n^{((j+1)\tau_l)}}=
    \sum\limits_{n'}\boldsymbol{Q}_{n',n}\boldsymbol{p}_{n'}^{[j]}.
    \label{trans}
\end{equation}
\color{black}
Define $\boldsymbol{P}^{[j]}=[\boldsymbol{p}_{1}^{[j]}, \boldsymbol{p}_{2}^{[j]}, \cdots, \boldsymbol{p}_{N}^{[j]}]$ as the label probability matrix, and the transition of $\boldsymbol{P}^{[j]}$ can be written as

\begin{equation}
    \boldsymbol{P}^{[j+1]}=\boldsymbol{P}^{[j]}\boldsymbol{Q}.
\end{equation}
Then the probability difference can be bounded \textit{w.r.t.} the mobility model.
\color{black}
\begin{lemma}\label{lemma1}
    If all vehicles follow a discrete Markov model with a transition probability matrix $\boldsymbol{Q}$, then the probability difference of edge server $n$ in $j\tau_l$-th iteration is bounded by:
    \begin{align}\label{Delta3}
        \lVert\boldsymbol{p}-\boldsymbol{p}^{[j]}_{n}\rVert_1\le NL_n\lVert\lambda_*\rVert^{j},
    \end{align}
    where $L_n$ is a constant determined by $n$, and $\lambda_*$ is the largest eigenvalue of $\boldsymbol{Q}$ satisfying $\lVert\lambda_*\rVert<1$.

\end{lemma}

\begin{proof}
    See Appendix \ref{app4}.
\end{proof}
\color{black}
Lemma \ref{lemma1} indicates that the probability difference decreases as the training goes on with a speed of $\lVert\lambda_*\rVert$. Taking Eq. \eqref{Delta2} into Eq. \eqref{Delta3}, we obtain
\begin{equation}
    \Delta^{[j]}\le GN\lVert\lambda_*\rVert^{j}\sum\limits_{n=1}^N \theta_n L_n=G\lVert\lambda_*\rVert^{j}NL,\label{Delta4}
\end{equation}
where $L=\sum\limits_{n=1}^N \theta_n L_n$. 

Now the mobility factor \textit{w.r.t.} several data distributions can be derived.
\color{black}
\begin{theorem}\label{theo3}
    For classification tasks, the mobility factors with different initial distributions have the following form:
    \begin{align}
        \hspace{-1mm}&\text{i.i.d.}\!:\! \gamma_k=0,\hspace{1mm}\label{mobfactor1.1}\\ 
        \hspace{-1mm}&\text{local non-i.i.d. / edge i.i.d.}\!:\! \gamma_k=-\frac{\eta\delta}{\tau_l\tau_e}H(\tau_l,\tau_e)<0, \hspace{1mm}\label{mobfactor1.2}\\ 
        \hspace{-1mm}&\text{edge non-i.i.d.}\!:\! \gamma_k\!=\!\frac{\eta GNL}{\tau_e}\!\sum\limits_{j=1}^{\tau_e-1}\!j\!\left(1\!-\!\lVert\lambda_*\rVert^{(k\!-\!1)\tau_e\!+\!j}\!\right)\!\!>\!0, \hspace{1mm}\label{mobfactor1.3}
    \end{align}
\end{theorem}

\color{black}
\begin{proof}
    For the i.i.d. case, since the local and edge data distributions are always identical to the global data distribution, we have $\delta=\Delta^{[j]}=0$ for all $j$, so Eq. \eqref{mobfactor1.1} is derived.
    
    For the local non-i.i.d. / edge i.i.d. case, we have $\delta\gg 0$ because the datasets of all vehicles are heterogeneous, and $\Delta^{[j]}=0$ for all $j$ because the datasets of edge servers are assumed to be i.i.d.. Since the upper bound of $\Delta^{[j]}$ decreases over $j$ as shown in Eq. \eqref{Delta4}, we can further assume edge i.i.d. always holds during training, i.e. $\Delta^{[j]}=0$ for all $j$. Taking it into Eq. \eqref{gamma}, Eq. \eqref{mobfactor1.2} is obtained.
    
    For the edge non-i.i.d. case, we assume that $\Delta^{[0]}-\Delta^{[(k\!-\!1)\tau_e+j]}\approx GNL(\lVert\lambda_*\rVert^{0}-\lVert\lambda_*\rVert^{j})$. Inserting \eqref{Delta4} into \eqref{gamma}, and based on Proposition \ref{prop2}, Eq. \eqref{mobfactor1.3} is obtained. Here $\gamma_k>0$ for all $k$ since $\lVert\lambda_*\rVert\le1$, so we have $\gamma>0$.
\end{proof}
\color{black}

\begin{figure*}[t!]
    \centering
    \begin{minipage}[t]{0.27\textwidth}
    \centering
        \includegraphics[width=\textwidth]{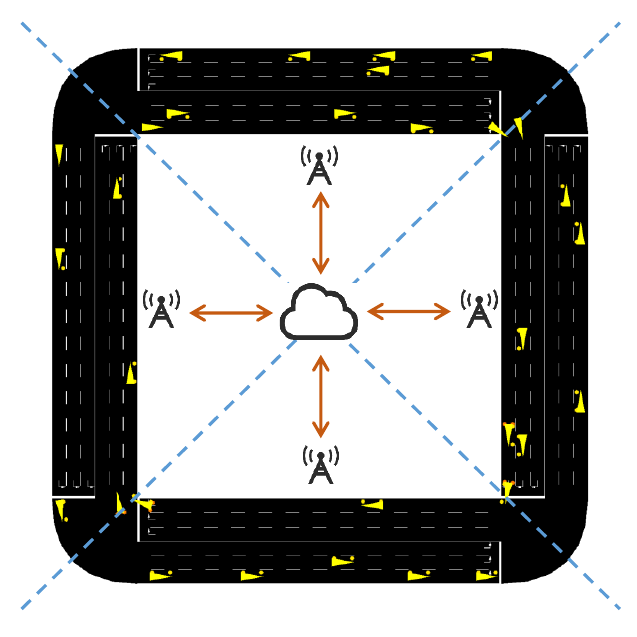}
        \caption{Simulation scenario: HFL with one cloud server, $N=4$ edge servers, and $M=32$ vehicles. The edge servers form a ring topology to compose a square.}
        \label{simulation}
    \end{minipage}
    \hfill
    \begin{minipage}[t]{0.315\textwidth}
    \centering
        \includegraphics[width=\textwidth]{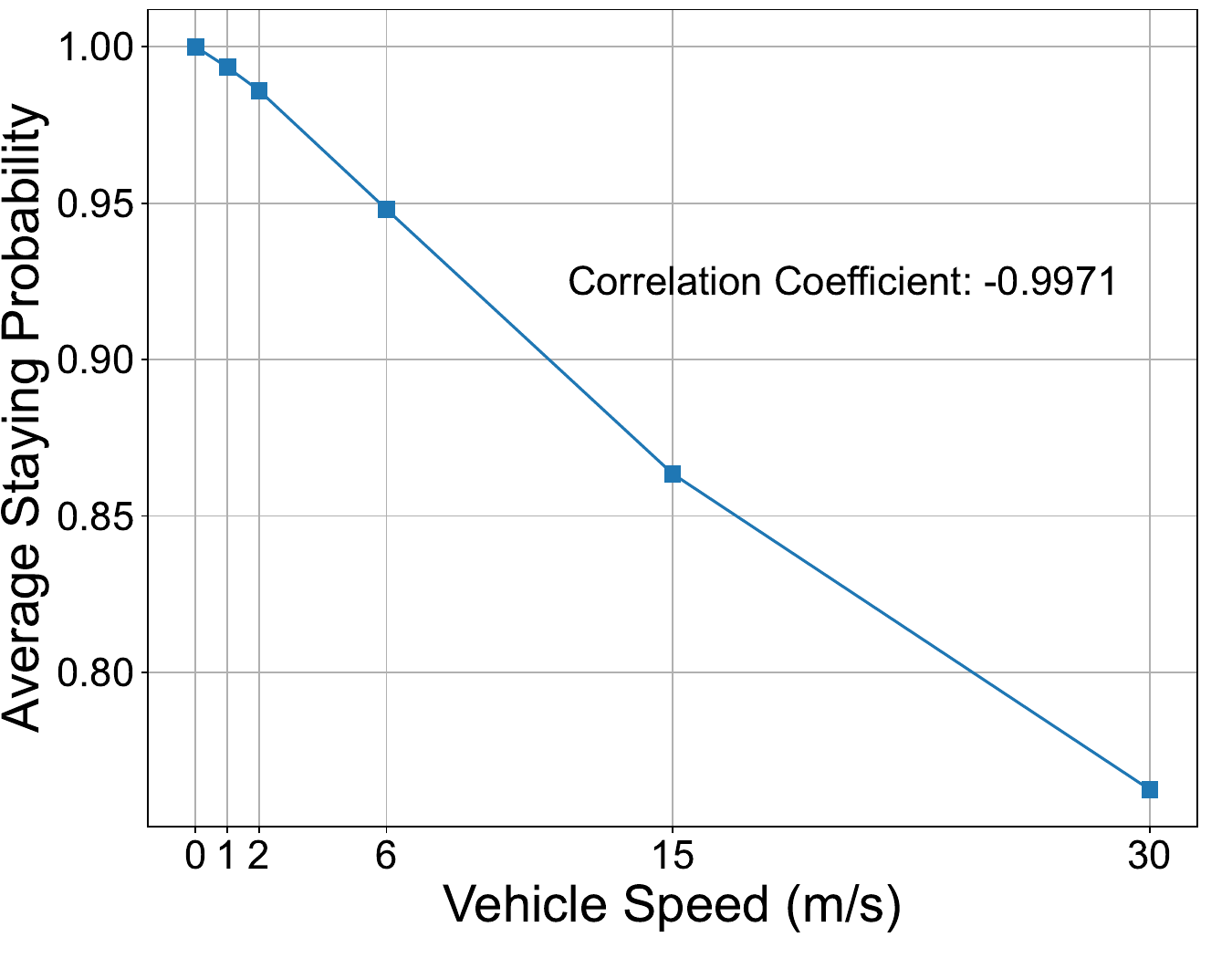}
        \caption{Average sojourn probability of vehicles with different vehicle speeds. The length of each side of the road is set to $a=1000$ m.}
        \label{fig1}
    \end{minipage}
    \hfill
    \begin{minipage}[t]{0.315\textwidth}
    \centering
        \includegraphics[width=\textwidth]{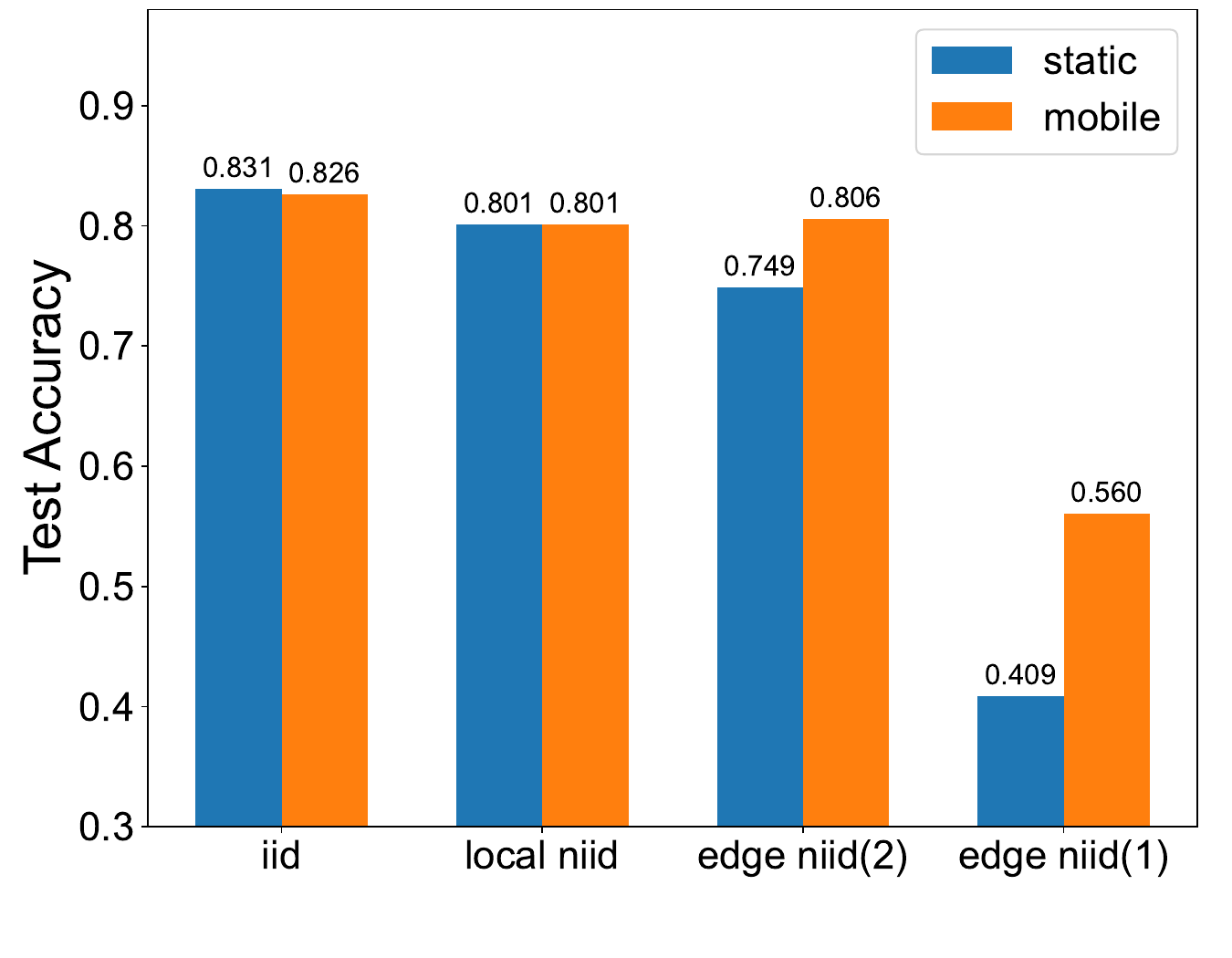}
        \caption{Comparison of the test accuracy within 600 cloud epochs for different initial data distribution.}
        \label{fig2}
    \end{minipage}
\end{figure*}

From Theorem \ref{theo3}, the impact of mobility on data-heterogeneous HFL depends on the initial data distribution. Mobility has no apparent influence on HFL with i.i.d. data, decreases the convergence speed of HFL with local non-i.i.d. / edge i.i.d. data by shuffling the edge models, and accelerates the convergence of HFL with edge non-i.i.d. data by fusing the edge data.

\subsubsection{Mobility Factor for Ring Topology}
$ $

Next, we want to further dig into the edge non-i.i.d. case to study the improvement brought by mobility in more detail. The mobility factor in Eq. \eqref{mobfactor1.3} connects the convergence bound of HFL with the eigenvalues of the transition probability matrix, but \textcolor{black}{the relation between the eigenvalues and the speed of vehicles is not clear}. In this part, a closed-form solution of eigenvalues is derived given specific topology and mobility model.

Assume edge servers are connected from head to tail to form a ring, and all edge servers are homogeneous in topology. For all the edge servers, the transition probabilities of vehicles is assumed to satisfy

\color{black}
    (1) Vehicles within any edge server have the same sojourn probability $p_s$,
    
    (2) For each transient step, vehicles within each edge server can only move to the adjacent edge servers with equal probability.
    
Under these assumptions, the transition probability matrix is denoted by
\color{black}
\begin{align}
    \boldsymbol{Q}=
    \begin{bmatrix}
      p_s & \frac{1-p_s}2 & 0 & \cdots & 0 & \frac{1-p_s}2 \\
      \frac{1-p_s}2 & p_s & \frac{1-p_s}2 & \cdots & 0 & 0\\
      0 & \frac{1-p_s}2 & p_s & \cdots & 0 & 0\\
      \vdots & \vdots & \vdots & \ddots & \vdots & \vdots\\
      0 & 0 & 0 & \cdots & p_s & \frac{1-p_s}2\\
      \frac{1-p_s}2 & 0 & 0 & \cdots & \frac{1-p_s}2 & p_s
    \end{bmatrix}.\label{circ matrix}
\end{align}
Then the closed-form solution of eigenvalues can be derived.
\begin{lemma}\label{lemma2}
    The transition probability matrix \eqref{circ matrix} is a circulant matrix with eigenvalues
    \begin{equation}
        \lambda_n=p_s+\left(1-p_s\right)\cos\left(\frac{2n\pi}{N}\right),\quad n=0,1,\cdots,N-1.\label{eigenvalue}
    \end{equation}
\end{lemma}

\begin{proof}
    From (3.2) in \cite{gray2006toeplitz}, the eigenvalues of circulant matrices take the following form
    \begin{equation}
        \lambda_n=\sum\limits_{k=0}^{N-1} c_k e^{\frac{-2\pi i k}{N}}, n=0,1,\cdots,N-1.\label{circ matrix2}
    \end{equation}
    Here $c_k$ denotes the $k$-th component in the first row of the matrix. Taking Eq. \eqref{circ matrix} into Eq. \eqref{circ matrix2}, Eq. \eqref{eigenvalue} is obtained.
\end{proof}
\begin{prop}
    For HFL with classification tasks, when the edge servers form a ring topology, the mobility factor in the edge non-i.i.d. case is expressed by 
    \begin{align}
    &\hspace{-2mm}\gamma_k=\frac{\eta GNL}{\tau_e}\sum\limits_{j=1}^{\tau_e-1}j\big[(1-p_s)\left(1-\cos\left(2\pi/N\right)\right)\big]^{k\tau_e+j}
   .\label{mobfactor2}
\end{align}
\end{prop}

\begin{proof}
    From \eqref{eigenvalue} we have $\lambda_*=p_s+\left(1-p_s\right)\cos\left(\frac{2\pi}{N}\right)$. Inserting it into \eqref{mobfactor1.3}, the proposition is proved.
\end{proof}

From \eqref{mobfactor2} we see that with a ring topology for the edge servers, the mobility factor is directly correlated to the sojourn probability of vehicles. Specifically, when vehicles move faster, their sojourn probability decreases, and the mobility factor $\gamma$ increases correspondingly, indicating the acceleration of the convergence speed of HFL. 

\section{Simulation Results}
\label{Sec-5}

In this section, we conduct simulations on HFL to show the influences of data heterogeneity and mobility and to verify our theoretical analysis.

\subsection{Simulation Settings} 
We consider a city road system on a square grid, where each side of the square is covered by an edge server, as shown in Fig. \ref{simulation}. There are $N=4$ edge servers and a total of $M=32$ vehicles in the system, and each edge server covers 8 vehicles initially. The initial positions of vehicles are uniformly distributed. The Simulation of Urban MObility (SUMO)\cite{lopez2018sumo} platform with the Manhattan model is adopted to simulate the mobility of vehicles, which monitors real-world vehicle dynamics. The length of each side of the road is $a$ m, and vehicles travel on the road at a maximum speed of $v$ m/s. Two main scenarios of the vehicle speed in simulations are considered:(1) the static scenario, where $v=0$; (2) the mobile scenario, where  $v=30$ m/s. The interval between each edge aggregation is assumed to be $1$ second. 
Since the edge servers follow the ring topology in our scenario, the mobility of vehicles is denoted by \eqref{circ matrix} with respect to the sojourn probability. As shown in Fig. \ref{fig1}, the average sojourn probability is negatively linearly correlated with the vehicle speed $v$ with a high correlation coefficient. Therefore, we can approximately change the sojourn probability by changing the vehicle speed with a linear scaling factor.

HFL simulations are conducted on the CIFAR-10 dataset\footnote{https://www.cs.toronto.edu/~kriz/cifar.html}. We consider three initial data distributions mentioned above: i.i.d. case, local non-i.i.d. case, and edge non-i.i.d. case. To create heterogeneous data by uniformly distributing the classes to edge servers, we only choose data of 8 classes from the dataset, with 40000 training samples and 8000 testing samples in total. For the i.i.d. case, the 40000 training samples are uniformly divided into $M$ disjoint subsets, and allocated to vehicles so that each vehicle owns one subset. For the local non-i.i.d. case, all the data are sorted by labels and $l$ classes are allocated to each vehicle (denoted by `local niid($l$)' in the following), with all vehicles holding the same number of samples. For the edge non-i.i.d. case, the sorted data is first allocated to each edge server so that each edge server owns $l$ classes of samples, and then the edge servers uniformly allocate their data to vehicles within their coverage area (denoted by `edge niid($l$)' in the following).

We train a convolutional neural network (CNN) with a batch size of 20 using the following architecture: the first convolutional block consists of two $3\times3$ convolution layers with 32 channels, ReLU activation, followed by a $2\times2$ max pooling layer and 0.2 dropout; the second convolutional block contains two $3\times3$ convolution layers with 64 channels, ReLU activation, followed by a $2\times2$ max pooling layer and 0.3 dropout; a fully connected layer with 120 units and ReLU activation; and a final 10-unit softmax layer. The network takes an input of size $20\times3\times32\times32$ and produces an output of size $20\times10$.  

The learning rate is set to $\eta=0.1$. If not especially noted, the local period and edge period are set to $\tau_l=6$ and $\tau_e=10$, separately. The Mob-HierFAVG algorithm is run for up to $K=600$ cloud epochs.

\begin{figure*}[!t]
	\centering
	\subfigure[Accuracy on the test dataset when training from scratch.]{\label{fig3}			
		\includegraphics[width=0.3\textwidth]{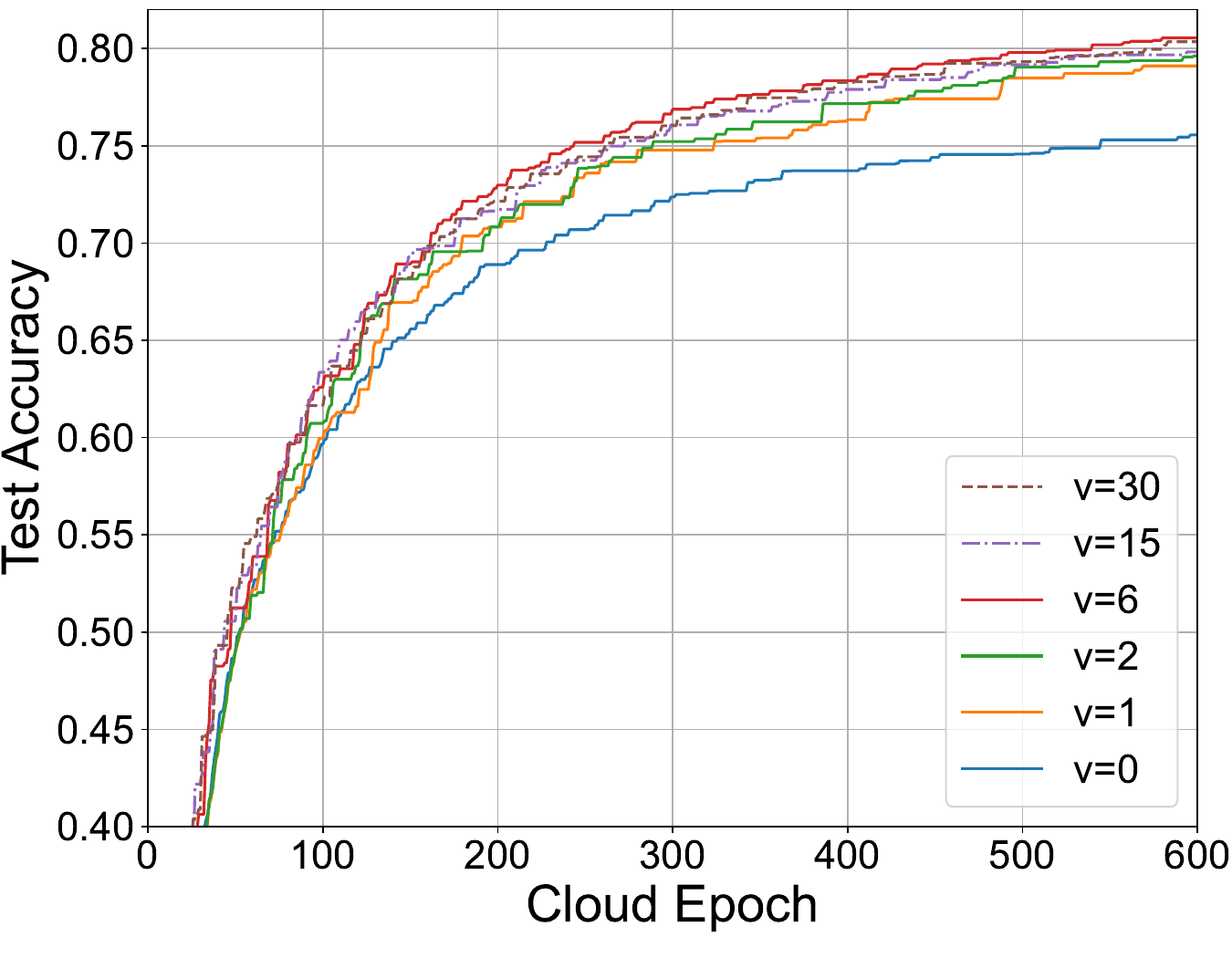} }		
	\hspace{33mm}
	\subfigure[Rounds to reach the target accuracies when training from scratch.]{\label{fig4}	
		\includegraphics[width=0.3\textwidth]{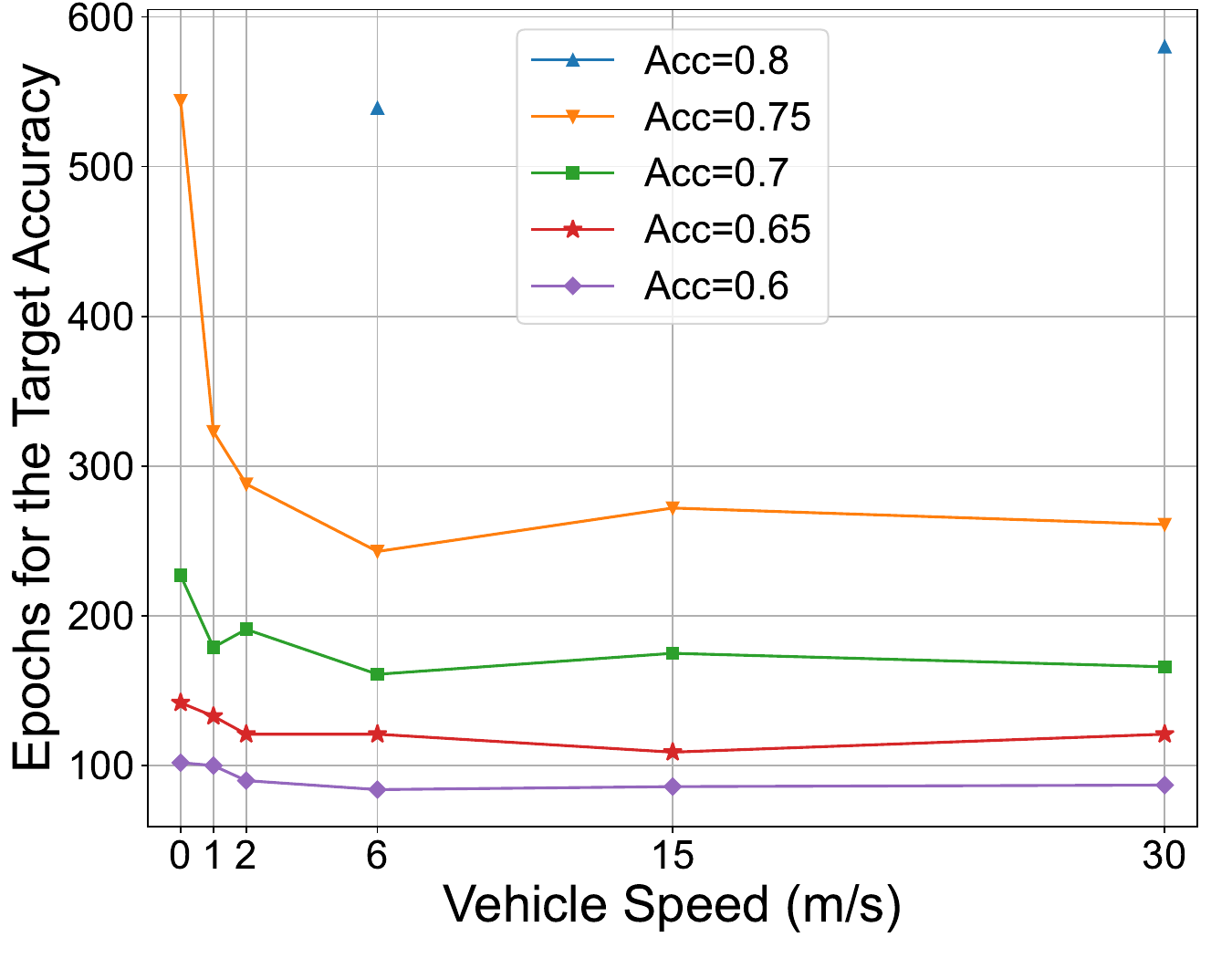} }
        \centering	
	\subfigure[Accuracy on the test dataset when training from a pretrained model with a test accuracy of 60\%.]{\label{fig5}			
		\includegraphics[width=0.3\textwidth]{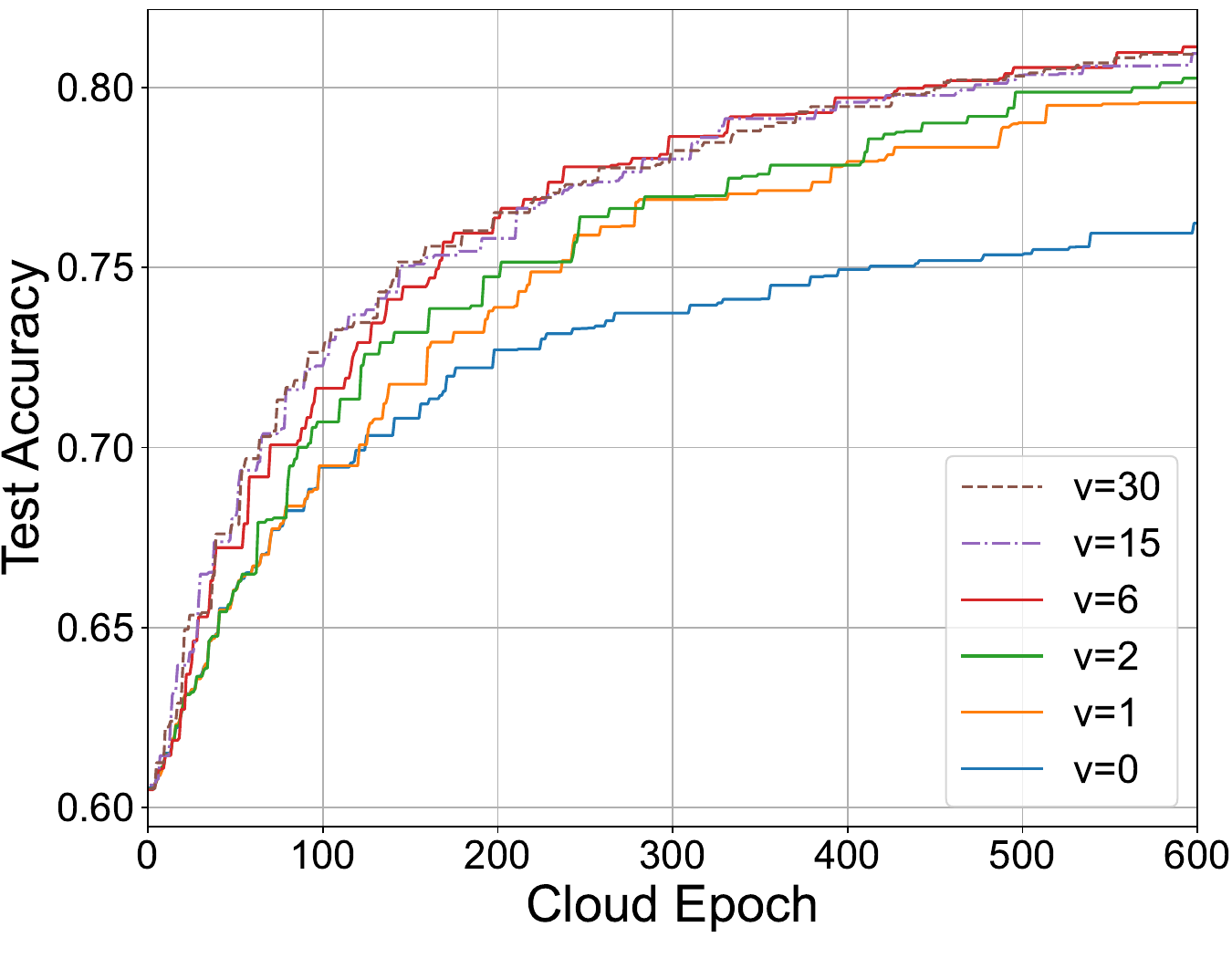} }	
        \hspace{33mm}
	\subfigure[Epochs to reach the target accuracies when training from a pretrained model with a test accuracy of 60\%.]{\label{fig6}	
		\includegraphics[width=0.3\textwidth]{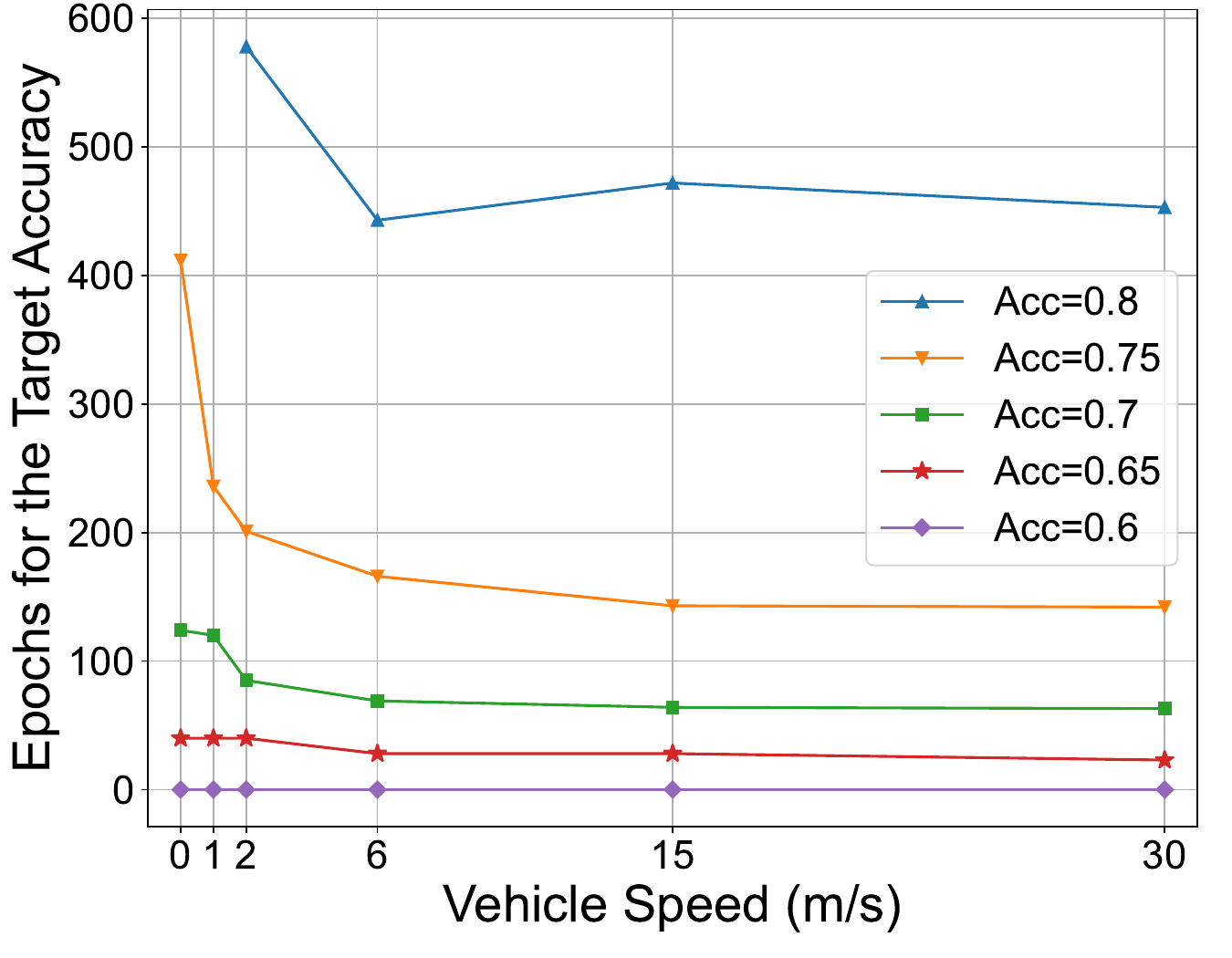} }
	\caption{Performance of HFL with different vehicle speeds.}
	\label{fig5-6}
\end{figure*}

    \begin{figure}[]
	\centering
	\includegraphics[width=0.3\textwidth]{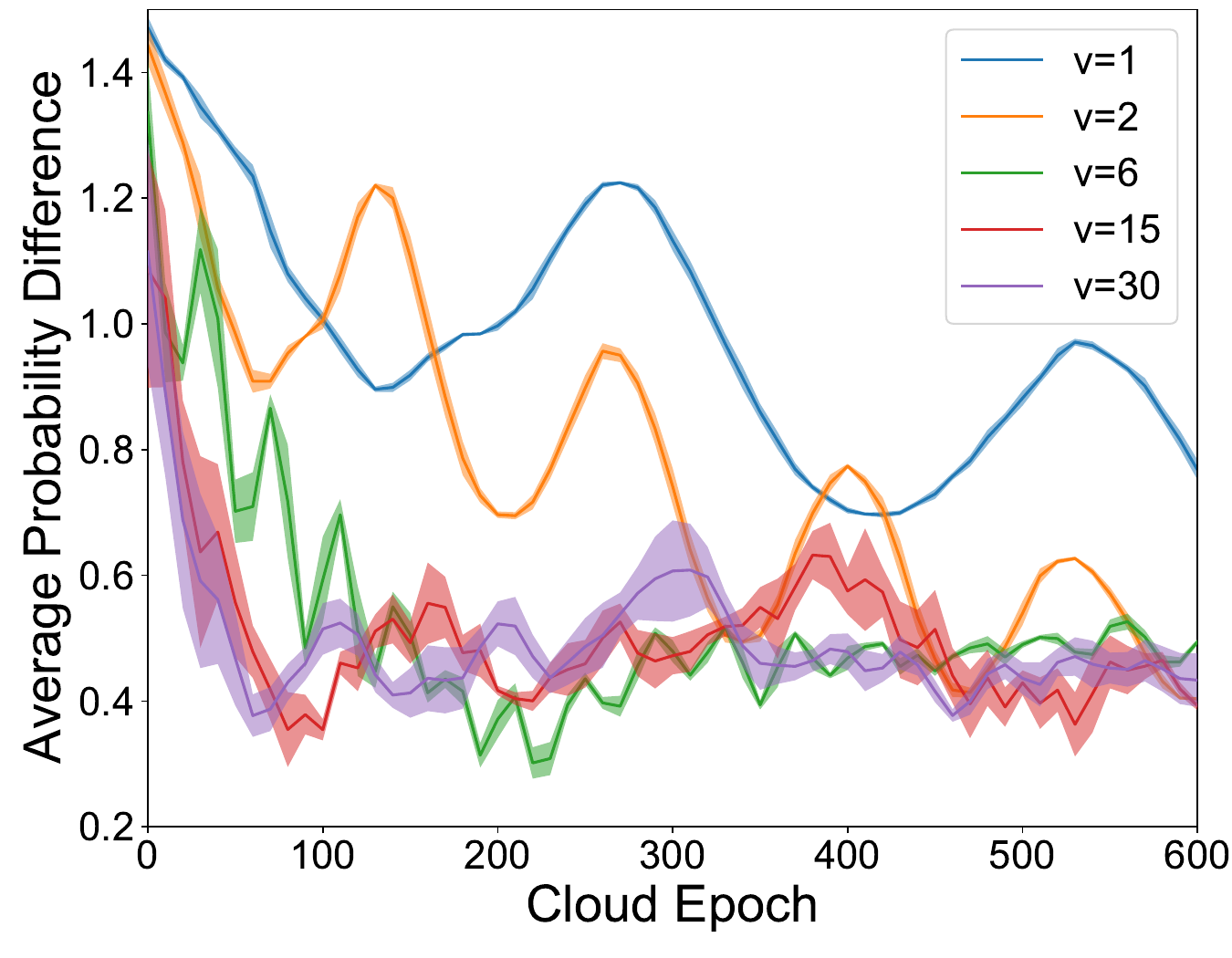}
	\caption{Variation of the average probability difference $\frac1N\sum_{n=1}^N \lVert\boldsymbol{p}-\boldsymbol{p}^{[j]}_{n}\rVert_1$ during training at different speeds. The values are calculated by averaging the results of 100 SUMO simulations. To make the curve smooth, the total 6000 points (600 cloud epochs $\times$ 10 edge aggregations/cloud epoch) are divided into 60 groups. The average value of each group is plotted and the standard deviation of each group is represented by the shadow around the curve.}
	\label{exp_12}
    \end{figure}
    
\subsection{Performance}
\begin{figure*}[!t]
	\centering
	\subfigure[Different $\tau_e$. ($\tau_l=6$)]{\label{fig7}			
		\includegraphics[width=0.3\textwidth]{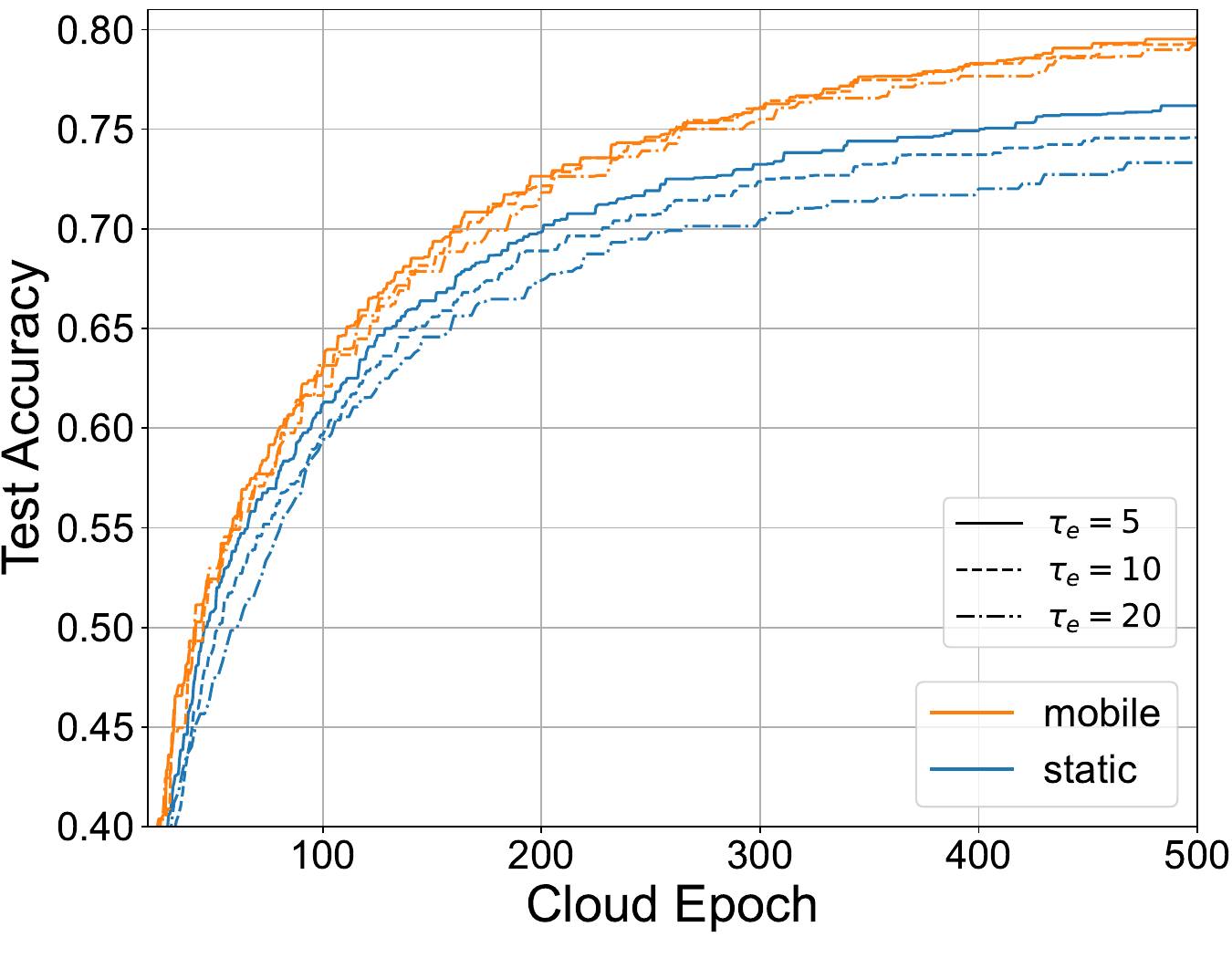} }		
	\subfigure[Different $\tau_l$. ($\tau_e=10$)]{\label{fig8}	
		\includegraphics[width=0.3\textwidth]{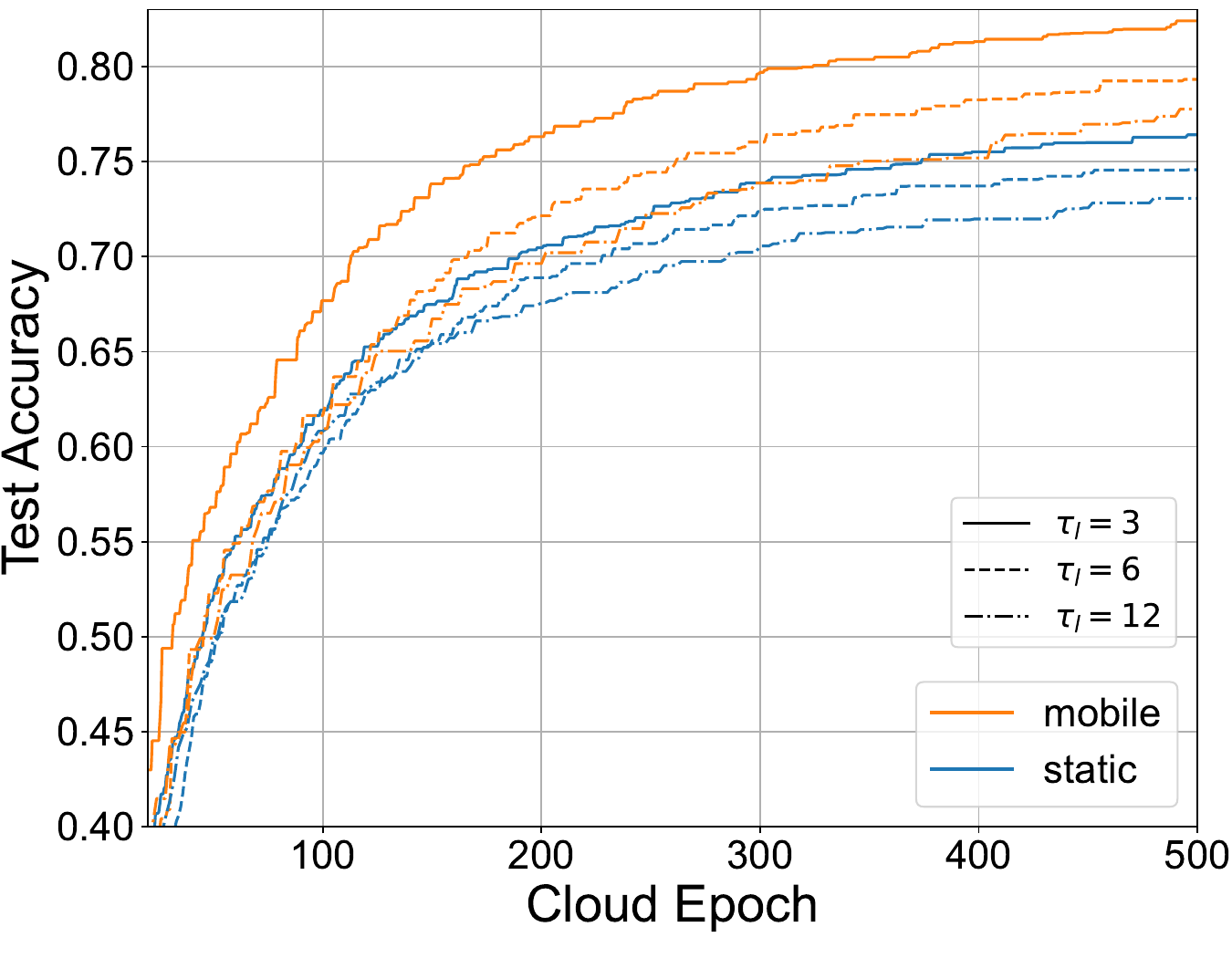} }
    \subfigure[Test accuracy within 600
        cloud epochs for different $\tau_e$ and $\tau_l$ with $\tau_e\tau_l=30$.]{\label{fig9}	
		\includegraphics[width=0.3\textwidth]{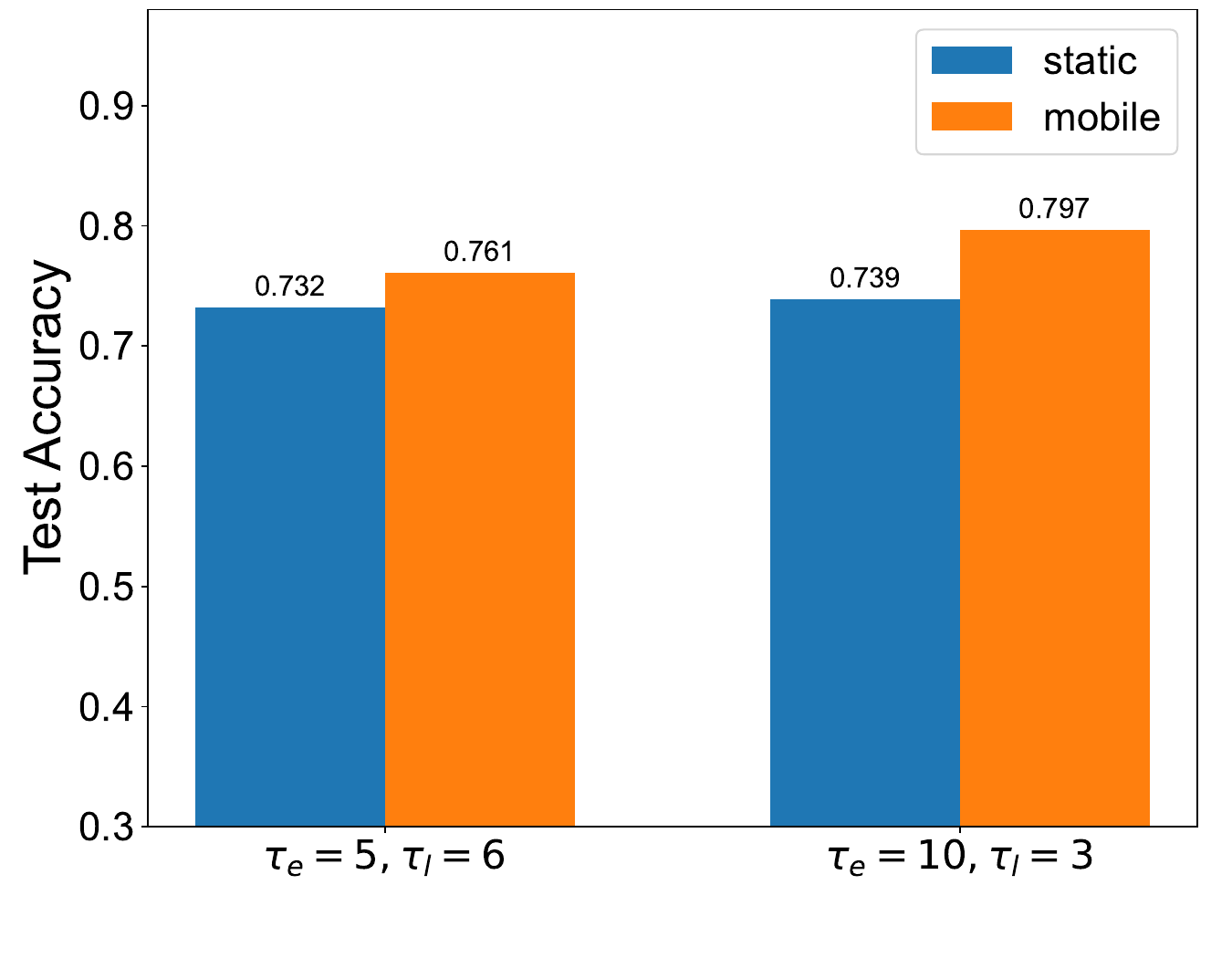} }
	\caption{Test accuracy of HFL with different aggregation periods in static and mobile scenarios.}
	\label{fig7-9}
    \end{figure*}
\subsubsection{Impact of initial data distribution}
    The performance of HFL under several initial data distributions is shown in Fig \ref{fig2}. For the i.i.d. case, the maximum test accuracies of the static and mobile scenarios are approximately equal, which is consistent with \eqref{mobfactor1.1}. For the local non-i.i.d. case, the static accuracy is the same as the mobile accuracy, which is different from the local non-i.i.d. / edge i.i.d. case in \eqref{mobfactor1.2}. This is because the number of simulated vehicles is not large enough to produce edge i.i.d., and as long as edge non-i.i.d. happens in any cloud epoch, the effect of data fusion emerges, which offsets the negative effect of model shuffling. For the edge non-i.i.d. case, mobility obviously increases the performance. Specifically, when $l=2$, mobility improves the accuracy by 5.7\% (74.9\% to 80.6\%); when $l=1$, mobility improves the accuracy by 15.1\% (40.9\% to 56.0\%). These results reflect that if an edge server holds fewer classes, the accuracy of HFL decreases, while the accuracy improvement brought by mobility increases. 
    
    Furthermore, by comparing the results of local non-i.i.d. and edge non-i.i.d. cases when $l=2$, we find that the local non-i.i.d. case achieves a higher accuracy than the edge non-i.i.d. case in the static scenario, but reaches a similar accuracy with mobility. This phenomenon may originate from the equivalent aggregation period of heterogeneous data. In the local non-i.i.d. case, the data heterogeneity exists at the local level, so the edge aggregation sees a fusion of heterogeneous data, and thus the equivalent aggregation period is $\tau_l$ local updates. In the edge non-i.i.d. case, the data heterogeneity exists at the edge level, so only the cloud aggregation sees a fusion of heterogeneous data with the equivalent aggregation period of $\tau_l\tau_e$ local updates. As demonstrated in \cite{li2021fedbn}, a greater aggregation period results in worse training performance for heterogeneous data. However, when the vehicles start moving, data of different edge servers gradually mix up, and thus the equivalent aggregation period goes back to $\tau_l$. 

    
    \subsubsection{Impact of vehicle speed}
    For the edge non-i.i.d. case, the impact of vehicle speed $v$ on the convergence speed is further studied. As is shown in Fig. \ref{fig3}, the test accuracies of the scenarios with mobility, i.e., $v>0$, are obviously higher than the one achieved when $v=0$. Furthermore, the rounds to reach certain target accuracies are shown in Fig. \ref{fig4}, which demonstrates that with high accuracy requirements, mobility greatly reduces the training epochs, and thus saves resource utilization and power consumption. Also, Fig. \ref{fig4} indicates that it takes nearly the same number of cloud epochs to reach the accuracy of 60\% for all scenarios. Models of this accuracy contain basic features of the datasets without the information of class variety and generally can be achieved by pre-training. Actually, this situation usually happens in vehicular networks. For example, for the beam selection task, when the channel condition changes, the model needs to be fine-tuned, and the original model, which carries the basic features for beams, can be set as the pre-trained model.

    Therefore, we further assume that the training is started from a pre-trained model with a test accuracy of 60\%. With a pre-trained model, the training time is obviously decreased, as shown in Fig. \ref{fig5} and \ref{fig6}. Besides, Fig. \ref{fig6} illustrates that a higher speed generally results in faster convergence. In particular, it takes only 142 cloud epochs for the $v=30$ case to reach an accuracy of 75\%, reduced by 65.5\% compared with the $v=0$ case (412 epochs), and 39.8\% compared with the $v=1$ case (236 epochs). These results are in line with the convergence analysis in Section \ref{Sec-4}.

    On the other hand, the improvement brought by high mobility is limited. In our simulations, both the final accuracy and the convergence speed typically saturate beyond $v=6$ $m/s$.
    Fig. \ref{exp_12} demonstrates that the evolving trends of 1-norm of the average probability difference of edges in $v=6,15,30$ $m/s$ cases are similar. Therefore, the convergence performance saturates possibly because data from different edge servers is sufficiently mixed at $v=6$ $m/s$, and no further gains can be achieved with even higher mobility. We also notice that the average probability difference does not monotonically decrease with time, but oscillates downward in a certain cycle. Besides, the theoretical analysis indicates that the average probability difference goes to zero after a long enough time, but during the simulation time, the value only decreases to around 0.4. These results may originate from the limited number of vehicles and training epochs.

        \begin{figure*}[!t]
	\centering
	\subfigure[Different number of edge servers. ($M=32$)]{\label{fig10}			
		\includegraphics[width=0.3\textwidth]{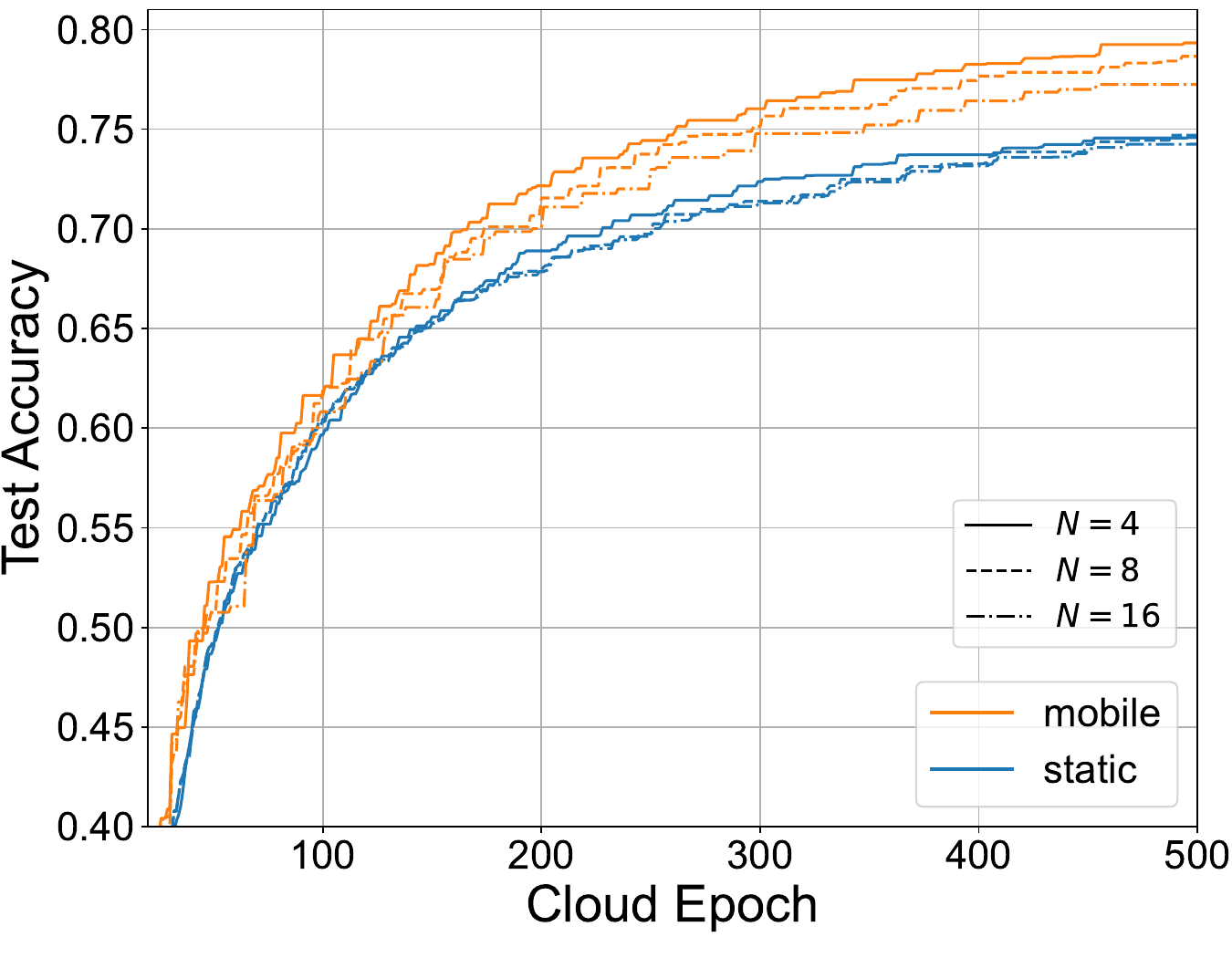} }		
	\hspace{20mm}
	\subfigure[Different number of vehicles. ($N=4$)]{\label{fig11}	
		\includegraphics[width=0.3\textwidth]{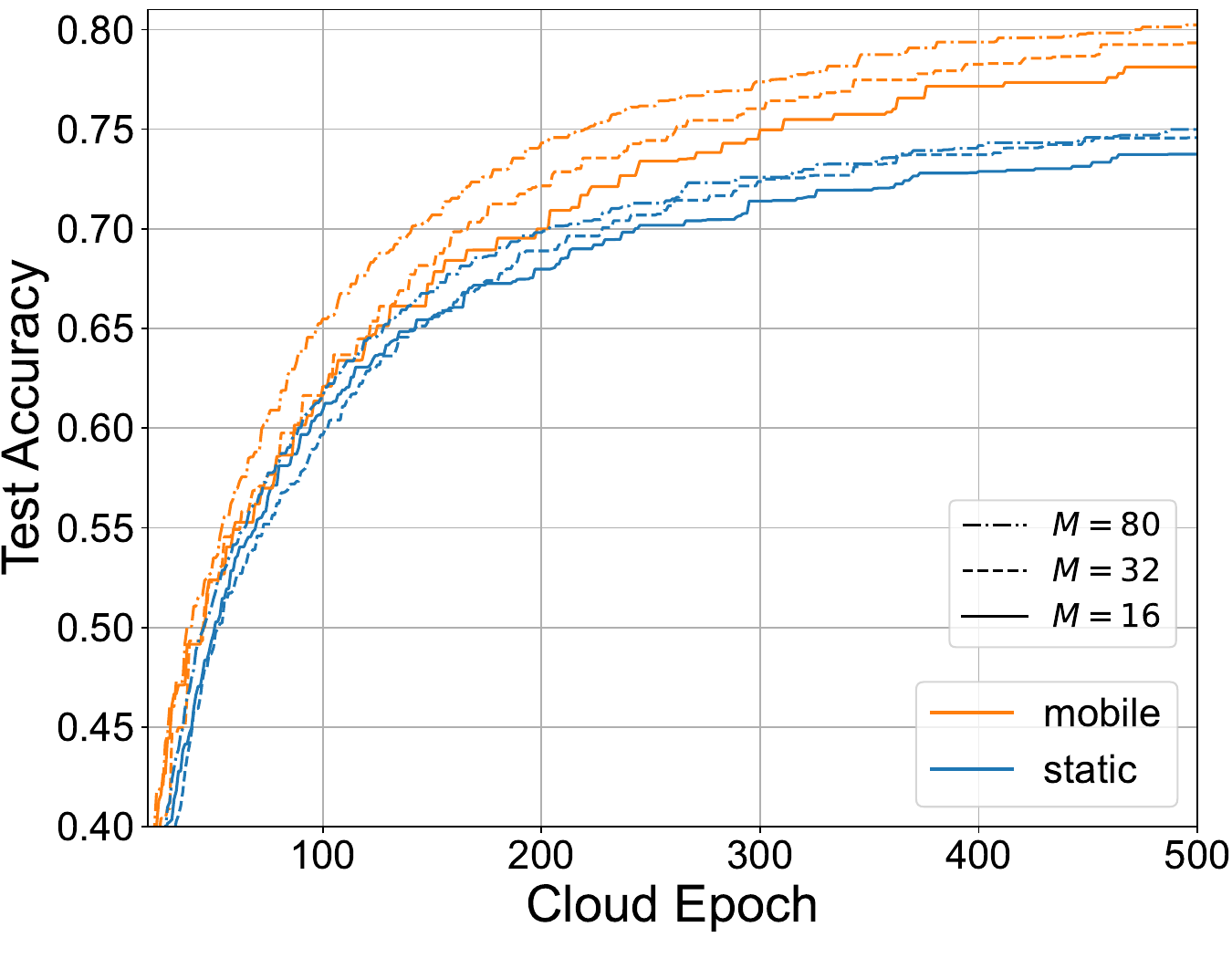} }
	\caption{Test accuracy of HFL with different numbers of vehicles for the static and mobile scenarios.}
	\label{fig10-11}
    \end{figure*}
    
    \subsubsection{Impact of aggregation periods}
    Eq. \eqref{mobfactor1.3} shows that the impact of mobility is correlated with not only the sojourn probability but also the aggregation periods and the number of vehicles. Further simulations are carried out for verification. Fig. \ref{fig7} shows the test accuracy of the static and mobile scenarios for several values of $\tau_e$. In the static scenario, a small $\tau_e$ results in faster convergence speed and higher accuracy; while in the mobile scenario, different values of $\tau_e$ report similar convergence speed and test accuracy values. This finding indicates that mobility reduces the reliance on cloud aggregation, and thus saves communication resources. 

    The impact of local period $\tau_l$ is shown in Fig. \ref{fig8}. Although not explicitly expressed in the mobility factor, mobility slightly enhances the benefits of edge aggregation. For the $\tau_l=12$ case, mobility improves the max test accuracy by 5.2\% (73.5\% to 78.7\%); while for the $\tau_l=3$ case, mobility improves the max test accuracy by 6.0\% (76.4\% to 82.4\%).

    From another perspective, if the local updates per cloud epoch $\tau_l\tau_e$ are fixed, we can learn the trade-off between the edge aggregation and cloud aggregation. Fig. \ref{fig9} tells that mobility brings greater improvement when edge aggregation is more frequent. 

    \subsubsection{Impact of the number of vehicles}
    As demonstrated in Fig. \ref{fig10}, the number of edge servers does not influence the convergence in the static scenario, but if the vehicles have mobility, fewer edge servers contribute to higher test accuracy. However, this result is based on the assumption that the cloud server has a certain coverage. If the number of edge servers decreases, each edge server needs to cover a larger area, with a longer communication delay during the edge aggregation. 

    Fig. \ref{fig11} demonstrates that more vehicles lead to higher test accuracy for both the static and mobile scenarios. This result basically matches the conclusion in FL: with more clients participating in the edge aggregation, the aggregated model is more stable, thus bringing a higher convergence speed. Besides, the improvement caused by adding vehicles is sharper for the mobile scenario, probably because with more vehicles, the probability difference of edges becomes smaller as the data of vehicles mix up.
\section{conclusion}
\label{Sec-6}
\color{black}
In this article, we have investigated the impact of mobility on data-heterogeneous HFL. The convergence analysis of HFL has been conducted, showing that mobility influences the performance of HFL by fusing the edge data and shuffling the edge models. Next, the mobility factors under three initial data distributions for classification tasks have been derived, which proves that mobility increases the convergence speed of HFL in the edge non-i.i.d. case. Furthermore, when the edge servers form a ring topology, a higher speed of vehicles lead to faster convergence of HFL. Simulations have been carried out on the CIFAR-10 dataset and the SUMO platform, demonstrating the benefits of mobility from multiple perspectives. Results have shown that mobility brings an increase in test accuracy by up to 15.1\% in the edge non-i.i.d. case. Higher speed results in higher test accuracy, while the achieved accuracy saturates when the speed of vehicles exceeds $6$ m/s. Moreover, mobility has also shown the potential for reducing training costs.
\color{black}
\appendices
\section{Proof of Proposition \ref{prop1}}\label{app1}
To prove Proposition \ref{prop1}, we first define 
    \begin{align}
        \phi(\tau)&=F\left(\boldsymbol{v}^{(\tau)}\right)-F\left(\boldsymbol{w}^*\right),\\
        \Tilde{\phi}(\tau)&=F\left(\Tilde{\boldsymbol{v}}^{(\tau)}\right)-F\left(\boldsymbol{w}^*\right).
    \end{align}
    As is proved in \cite{wang2019adaptive}, we have $\phi(\tau), \Tilde{\phi}(\tau)>0$ and $\phi(\tau)>\Tilde{\phi}\left(\tau+1\right)$. Besides, according to the definition, we obtain
    \begin{align}
        \phi(\tau)&=\Tilde{\phi}(\tau),\tau_l\tau_e\nmid\tau\label{phi},\\
        \phi(\tau)&\ge\Tilde{\phi}(\tau),\tau_l\tau_e\mid\tau.
    \end{align}
    Next, we prove the following lemma.

    \begin{lemma}\label{lemma3}
        For any $k$, when $\eta\ge\frac1{\beta}$ , we have
        \begin{equation}
            \frac1{\Tilde{\phi}\left(\tau+1\right)}-\frac1{\phi(\tau)}\ge\omega\eta\left(1-\frac{\beta\eta}{2}\right),
        \end{equation}
        where $\omega=\min\limits_k\frac1{\left\lVert{\Tilde{\boldsymbol{v}}^{\left(\left(k-1\right)\tau_l\tau_e\right)}-\boldsymbol{w}^*}\right\rVert^2}$.
    \end{lemma}
    
\begin{proof}
        From Lemma 5 in \cite{wang2019adaptive}, we have
        \begin{equation}
            \Tilde{\phi}\left(\tau+1\right)-\phi(\tau)\le-\eta\left(1-\frac{\beta\eta}2\right)\left\lVert{\nabla F\left(\Tilde{\boldsymbol{v}}^{(\tau)}\right)}\right\rVert^2.
        \end{equation}
        Equivalently,
        \begin{equation}\label{lemma1_1}
            \Tilde{\phi}\left(\tau+1\right)\le\phi(\tau)-\eta\left(1-\frac{\beta\eta}2\right)\left\lVert{\nabla F\left(\Tilde{\boldsymbol{v}}^{(\tau)}\right)}\right\rVert^2.
        \end{equation}
        Furthermore, the convexity condition and Cauchy-Schwarz inequality give that
        \begin{align}
            \phi(\tau)&=F\left(\Tilde{\boldsymbol{v}}^{(\tau)}\right)-F\left(\boldsymbol{w}^*\right)\\
            &\le \nabla F\left(\Tilde{\boldsymbol{v}}^{(\tau)}\right)^T\left(\Tilde{\boldsymbol{v}}^{(\tau)}-\boldsymbol{w}^*\right)\\
            &\le\left\lVert{\nabla F\left(\Tilde{\boldsymbol{v}}^{(\tau)}\right)}\right\rVert \left\lVert{\Tilde{\boldsymbol{v}}^{(\tau)}-\boldsymbol{w}^*}\right\rVert.\label{lemma1_2}
        \end{align}
        Take \eqref{lemma1_1} into \eqref{lemma1_2}, and we get
        \begin{equation}
            \Tilde{\phi}\left(\tau+1\right)\le\phi(\tau)-\eta\left(1-\frac{\beta\eta}2\right)\frac{\phi(\tau)^2}{\left\lVert{\Tilde{\boldsymbol{v}}^{(\tau)}-\boldsymbol{w}^*}\right\rVert^2}.
        \end{equation}
        Dividing both sides by $\phi(\tau)\Tilde{\phi}\left(\tau+1\right)$, we obtain
        \begin{align}
            \frac1{\phi(\tau)} 
            &\le
            \frac1{\Tilde{\phi}\left(\tau+1\right)}-\eta\left(1-\frac{\beta\eta}2\right)\frac{\phi(\tau)}{\Tilde{\phi}\left(\tau+1\right)\left\lVert{\Tilde{\boldsymbol{v}}^{(\tau)}-\boldsymbol{w}^*}\right\rVert^2}\\
            &\le 
            \frac1{\Tilde{\phi}\left(\tau+1\right)}-\eta\left(1-\frac{\beta\eta}2\right)\frac{1}{\left\lVert{\Tilde{\boldsymbol{v}}^{(\tau)}-\boldsymbol{w}^*}\right\rVert^2}\\
            &\le 
            \frac1{\Tilde{\phi}\left(\tau+1\right)}-\omega\eta\left(1-\frac{\beta\eta}2\right)\label{lemma1_3},
        \end{align}
        where the second inequality comes from (\ref{phi}), and the third inequality holds because $\left\lVert{\Tilde{\boldsymbol{v}}^{\left(\tau+1\right)}-\boldsymbol{w}^*}\right\rVert^2
        \le
        \left\lVert{\Tilde{\boldsymbol{v}}^{(\tau)}-\boldsymbol{w}^*}\right\rVert^2$ when $\tau\in[\left(k-1\right)\tau_l\tau_e,k\tau_l\tau_e)$. Arranging \eqref{lemma1_3} we prove Lemma \ref{lemma3}.
\end{proof}

Now we can focus on proving Proposition \ref{prop1}.

        Using Lemma \ref{lemma3} and considering $\tau\in[0,K\tau_l\tau_e]$, we have
        
        \begin{align}
            &\sum\limits_{\tau=1}^{K\tau_l\tau_e}\frac1{\Tilde{\phi}\left(\tau+1\right)}-\frac1{\phi(\tau)}\\
            =&\sum\limits_{k=1}^K\Big[\frac1{\Tilde{\phi}\left(k\tau_l\tau_e\right)}-\frac1{\phi\left(\left(k-1\right)\tau_l\tau_e\right)}\Big]\\
            =&\frac1{\Tilde{\phi}\left(K\tau_l\tau_e\right)}-\frac1{\phi\left(0\right)}-\sum\limits_{k=1}^{K-1}\Big[\frac1{\phi\left(k\tau_l\tau_e\right)}-\frac1{\Tilde{\phi}\left(k\tau_l\tau_e\right)}\Big]\label{dcdiff}\\
            \ge& K\tau_l\tau_e\omega\eta\left(1-\frac{\beta\eta}2\right),
        \end{align}
        where the first equality holds because $\Tilde{\phi}(\tau)=\phi(\tau)$ when $\tau_l\tau_e\nmid\tau$ according to the definition. Furthermore, each term in the sum of \eqref{dcdiff} can be expressed as
        \begin{align}
            \frac1{\phi\left(k\tau_l\tau_e\right)}-\frac1{\Tilde{\phi}\left(k\tau_l\tau_e\right)}
            &=
            \frac{\Tilde{\phi}\left(k\tau_l\tau_e\right)-\phi\left(k\tau_l\tau_e\right)}{\Tilde{\phi}\left(k\tau_l\tau_e\right)\phi\left(k\tau_l\tau_e\right)}\\
            &=
            \frac{F\left(\Tilde{\boldsymbol{v}}^{\left(k\tau_l\tau_e\right)}\right)-F\left(\boldsymbol{v}^{\left(k\tau_l\tau_e\right)}\right)}{\Tilde{\phi}\left(k\tau_l\tau_e\right)\phi\left(k\tau_l\tau_e\right)}\\
            &\ge 
            \frac{-\rho U_k}{\Tilde{\phi}\left(k\tau_l\tau_e\right)\phi\left(k\tau_l\tau_e\right)}\label{uphi},
        \end{align}
        where the inequality holds because of Assumption \ref{assu1}. Combining condition (3) in Proposition \ref{prop1} with \eqref{phi}, the denominator in the right-hand side of \eqref{uphi} can be bounded by
        \begin{align}
            \Tilde{\phi}\left(k\tau_l\tau_e\right)\phi\left(k\tau_l\tau_e\right)\ge
            \Tilde{\phi}^2\left(k\tau_l\tau_e\right)\ge\epsilon^2.\label{phiep}
        \end{align}
        Substituting \eqref{phiep} into \eqref{dcdiff}, we obtain
        \begin{equation}
            \frac1{\Tilde{\phi}\left(K\tau_l\tau_e\right)}-\frac1{\phi\left(0\right)}\ge
            K\tau_l\tau_e\omega\eta\left(1-\frac{\beta\eta}2\right)-\sum\limits_{k=1}^{K-1}\frac{\rho U_k}{\epsilon^2}.\label{kphi}
        \end{equation}
        Furthermore, from \eqref{uphi} and \eqref{phiep} we also have
        \begin{align}
            \frac1{\phi\left(K\tau_l\tau_e\right)} - \frac1{\Tilde{\phi}\left(K\tau_l\tau_e\right)}\ge-\frac{\rho U_K}{\epsilon^2}.\label{Kphi}
        \end{align}
        Combining \eqref{kphi} with \eqref{Kphi} and condition (2), we obtain
        \begin{align}
            \frac1{\phi\left(K\tau_l\tau_e\right)}
            &\ge\frac1{\phi\left(K\tau_l\tau_e\right)}-\frac1{\phi\left(0\right)}\\
            &\ge
            K\tau_l\tau_e\omega\eta\left(1-\frac{\beta\eta}2\right)-\sum\limits_{k=1}^{K}\frac{\rho U_k}{\epsilon^2}>0.\label{finalphi}
        \end{align}
        where the first inequality holds because $\phi(\tau)>0$ for any $\tau$.
        Since $F\left(\boldsymbol{v}^{\left(K\tau_l\tau_e\right)}\right)=F\left(\boldsymbol{u}^{\left(K\tau_l\tau_e\right)}\right)=F\left(\boldsymbol{w}^{\left(T\right)}\right)$, we rearrange \eqref{finalphi} to get
        \begin{align}
            \phi\left(K\tau_l\tau_e\right)
            &=F\left(\boldsymbol{w}^{\left(T\right)}\right)-F\left(\boldsymbol{w}^*\right)\le \frac1{T\omega\eta\left(1-\frac{\beta\eta}2\right)-\sum\limits_{k=1}^{K}\frac{\rho U_k}{\epsilon^2}}\\[-10pt]
            &=
            \frac1{T\eta\varphi-\sum\limits_{k=1}^{K}\frac{\rho U_k}{\epsilon^2}}=
            \frac1{T\left(\eta\varphi-\frac1K\sum\limits_{k=1}^{K}\frac{\rho U_k}{\tau_l\tau_e\epsilon^2}\right)}.\label{convergence}
        \end{align}


    \section{proof of Theorem \ref{theo1}}\label{app2}
        
    Denote the virtual edge parameter $\boldsymbol{u}^{(\tau)}_n$ as
    \begin{equation}
        \boldsymbol{u}^{(\tau)}_n=\sum_{m\in \mathcal{E}_n^{(\tau)}}\alpha^{(\tau)}_{m,n} \boldsymbol{w}_{m}^{(\tau)}.
    \end{equation}
    
    To prove Theorem \ref{theo1}, we first propose some lemmas.
\begin{lemma}
for any $\tau$, we have
\begin{flalign}
    &\left\lVert{\boldsymbol{u}^{(\tau)}\!-\!\boldsymbol{\Tilde{v}}^{(\tau)}}\right\rVert\le&
\end{flalign}\vspace{-15pt}
\begin{align}
\begin{cases}
\left\lVert{\boldsymbol{u}^{(\tau-1)}\!-\!\boldsymbol{v}^{(\tau-1)}}\right\rVert
\!+\!\eta\beta \sum_m\alpha_m \left\lVert{\boldsymbol{w}_m^{(\tau-1)}\!-\!\boldsymbol{v}^{(\tau-1)}}\right\rVert,\\
\hspace{13em}\tau_l\nmid\tau\!-\!1, \\
\left\lVert{\boldsymbol{u}^{(\tau-1)}\!-\!\boldsymbol{v}^{(\tau-1)}}\right\rVert
\!+\!\eta\beta\hspace{-1pt} \sum_n\!\theta_n^{(\tau-1)} \!\left\lVert{{\boldsymbol{u}}_{n}^{(\tau-1)}\!-\!\boldsymbol{v}^{(\tau-1)}}\right\rVert,\\
\hspace{13em}\tau_l\mid\tau\!-\!1,\tau_l\tau_e\nmid\tau\!-\!1, \\
0, \hspace{12.05em}\tau_l\tau_e\mid\tau\!-\!1,
\end{cases}\label{clouddiff}
\end{align}\vspace{-6pt}

\end{lemma}
\begin{proof}
When $\tau_l\tau_e\mid \tau-1$, let $\tau=k\tau_l\tau_e+1$, and thus
\begin{align}
    \left\lVert {\boldsymbol{u}^{(\tau)}-\Tilde{\boldsymbol{v}}^{(\tau)}} \right\rVert 
    =&\bigg\lVert \bigg[\boldsymbol{u}^{\left(k\tau_l\tau_e\right)}\!-\!
    \eta\sum\limits_m\alpha_m\nabla f_m\!\left(\! \boldsymbol{w}_m^{\left(k\tau_l\tau_e\right)}\!\right)\!\bigg]\\
    &-\bigg[\boldsymbol{v}^{\left(k\tau_l\tau_e\right)}-\eta\nabla F\!\left(\!\boldsymbol{v}^{\left(k\tau_l\tau_e\right)}\!\right)\!\bigg] \bigg\rVert\\
    =&\bigg\lVert \bigg[\boldsymbol{u}^{\left(k\tau_l\tau_e\right)}\!-\!\eta\nabla F\!\left(\!\boldsymbol{u}^{\left(k\tau_l\tau_e\right)}\!\right)\!\bigg]
    \\&-\bigg[\boldsymbol{v}^{\left(k\tau_l\tau_e\right)}\!-\!\eta\nabla F\!\left(\!\boldsymbol{v}^{\left(k\tau_l\tau_e\right)}\!\right)\!\bigg] \bigg\rVert\\
    =&0.
\end{align}
where the second equality holds because $\boldsymbol{w}_m^{\left(k\tau_l\tau_e\right)}=\boldsymbol{u}^{\left(k\tau_l\tau_e\right)}$ for all $m$ and $\sum\limits_m\alpha_m f_m\left(\cdot\right)=F\left(\cdot\right)$, and the last equality holds because $\boldsymbol{u}^{\left(k\tau_l\tau_e\right)}=\boldsymbol{v}^{\left(k\tau_l\tau_e\right)}$ according to the definition.

When $\tau_l\mid\tau-1$ and $\tau_l\tau_e\nmid\tau-1$ , we have
\begin{align*}
&\left\lVert \boldsymbol{u}^{(\tau)}-\Tilde{\boldsymbol{v}}^{(\tau)}\right\rVert \vspace{2ex}\\
=&\bigg\lVert \Big[
\boldsymbol{u}^{(\tau-1)}-\eta\sum\limits_n\theta_n\nabla F_n\left( \boldsymbol{w}_n^{(\tau-1)}\right)\Big]\\
&
-\left[
\boldsymbol{v}^{(\tau-1)}-\eta\nabla F\left(\boldsymbol{v}^{(\tau-1)}\right)\right] \bigg\rVert \vspace{2ex}\\
=&\bigg\lVert{\left[
\boldsymbol{u}^{(\tau-1)}-\boldsymbol{v}^{(\tau-1)}\right]}\\
&{-\eta \sum\limits_n
\theta_n\left[\nabla F_n\left( \boldsymbol{w}_n^{(\tau-1)}\right)-\nabla F_n\left(\boldsymbol{v}^{(\tau-1)}\right)\right]} \bigg\rVert \\
\le
&\left\lVert
\boldsymbol{u}^{(\tau-1)}-\boldsymbol{v}^{(\tau-1)} \right\rVert\\
&+\eta \sum\limits_n
\theta_n\left\lVert\nabla F_n\left( \boldsymbol{w}_n^{(\tau-1)}\right)-\nabla F_n\left(\boldsymbol{v}^{(\tau-1)}\right) \right\rVert \\
\le&\left\lVert{\boldsymbol{u}^{(\tau-1)}-\boldsymbol{v}^{(\tau-1)}}\right\rVert+\eta\beta \sum_n\theta_n \left\lVert{\boldsymbol{w}_n^{(\tau-1)}-\boldsymbol{v}^{(\tau-1)}}\right\rVert\label{try}.
\end{align*}

Similarly, when $\tau_l\nmid\tau-1$, we obtain \eqref{try}
\begin{align}
    &\left\lVert \boldsymbol{u}^{(\tau)}-\Tilde{\boldsymbol{v}}^{(\tau)}\right\rVert  \vspace{2ex}\\
    =&\bigg\lVert \left[
        \boldsymbol{u}^{(\tau-1)}-\boldsymbol{v}^{(\tau-1)}\right]\\
    &-\eta \sum\limits_m
    \alpha_m\left[\nabla f_m\left( \boldsymbol{w}_m^{(\tau-1)}\right)-\nabla f_m\left(\boldsymbol{v}^{(\tau-1)}\right)\right] \bigg\rVert \\
    \le&\left\lVert 
        \boldsymbol{u}^{(\tau-1)}-\boldsymbol{v}^{(\tau-1)} \right\rVert\\
    &+\eta \sum\limits_m
    \alpha_m\left\lVert\nabla f_m\left( \boldsymbol{w}_m^{(\tau-1)}\right)-\nabla f_m\left(\boldsymbol{v}^{(\tau-1)}\right)\right\rVert \\
    \le&\left\lVert{\boldsymbol{u}^{(\tau-1)}-\boldsymbol{v}^{(\tau-1)}}\right\rVert+\eta\beta \sum_m\alpha_m \left\lVert\boldsymbol{w}_m^{(\tau-1)}-\boldsymbol{v}^{(\tau-1)}\right\rVert.
\end{align}

\end{proof}

\begin{lemma}
    
For the virtual edge parameter, we have
\begin{align}
    &\left\lVert {\boldsymbol{u}_n^{(\tau)}-\Tilde{\boldsymbol{v}}^{(\tau)}} \right\rVert \le\label{edge_cent_diff}\\ 
    &\begin{cases} 
    \left(1+\eta\beta\right)\left\lVert {\boldsymbol{u}_n^{(\tau-1)}-\boldsymbol{v}^{(\tau-1)}}\right\rVert+\eta\Delta_n, \tau_l\mid\tau-1\\
    \left\lVert {\boldsymbol{u}_n^{(\tau-1)}-\boldsymbol{v}^{(\tau-1)}} \right\rVert \\+ \eta\beta\sum\limits_{m\in \mathcal{E}_n}\alpha_{m,n}
    \left\lVert  {\boldsymbol{w}_m^{(\tau-1)}-\boldsymbol{v}^{(\tau-1)}}\right\rVert+\eta\Delta_n, \tau_l\nmid\tau-1.
    \end{cases}\label{unv}
\end{align}
\end{lemma}

\begin{proof}
When $\tau_l\mid\tau-1$, we can decompose \eqref{edge_cent_diff} by
\begin{align}
    &\left\lVert {\boldsymbol{u}_n^{(\tau)}-\Tilde{\boldsymbol{v}}^{(\tau)}} \right\rVert\\
    =&\Big\lVert \Big[
        \boldsymbol{u}_n^{(\tau-1)}\!-\!\eta\nabla F_n\left( \boldsymbol{u}_n^{(\tau-1)}\right)\Big]\!-\!\Big[
        \boldsymbol{v}^{(\tau-1)}\!-\!\eta\nabla F\left(\boldsymbol{v}^{(\tau-1)}\right)\Big] \Big\rVert\\
    =&\Big\lVert 
    \Big[
        \boldsymbol{u}_n^{(\tau-1)}\!-\!\boldsymbol{v}^{(\tau-1)}\Big]
    \!-\!\Big[
        \eta\nabla F_n\left( \boldsymbol{u}_n^{(\tau-1)}\right)\!-\!\eta\nabla F_n\left(\boldsymbol{v}^{(\tau-1)}\right)\Big]\\
    &- \Big[\eta\nabla
    F_n\left(\boldsymbol{v}^{(\tau-1)}\right)-\eta\nabla
    F\left(\boldsymbol{v}^{(\tau-1)}\right)\Big]
     \Big\rVert\\ 
    \le&\left\lVert {
    \boldsymbol{u}_n^{(\tau-1)}\!-\!\boldsymbol{v}^{(\tau-1)}}\right\rVert
    \!+\!\eta\left\lVert 
    \nabla F_n\!\left(\! \boldsymbol{u}_n^{(\tau-1)}\!\right)\!-\!\nabla F_n\!\left(\!\boldsymbol{v}^{(\tau-1)}\!\right)\!\right\rVert
    \!+\!\eta\Delta_n\\
    =&\left(1+\eta\beta\right)\left\lVert {
    \boldsymbol{u}_n^{(\tau-1)}-\boldsymbol{v}^{(\tau-1)}}\right\rVert+\eta\Delta_n.
\end{align}
Similarly, when $\tau_l\nmid\tau-1$, we have
\begin{align}
    &\left\lVert {\boldsymbol{u}_n^{(\tau)}-\Tilde{\boldsymbol{v}}^{(\tau)}} \right\rVert\\
    =&\Big\lVert 
    \Big[
        \boldsymbol{u}_n^{(\tau-1)}\!-\!\boldsymbol{v}^{(\tau-1)}\Big]
    \!-\! \Big[\eta\nabla
    F_n\left(\boldsymbol{v}^{(\tau-1)}\right)\!-\!\eta\nabla
    F\left(\boldsymbol{v}^{(\tau-1)}\right)\Big]\\
    &\!-\!\eta\!\sum\limits_{m\in\mathcal{E}_n}\alpha_{m,n}\Big[
        \nabla f_m\left(\boldsymbol{w}_m^{(\tau-1)}\right)\!-\!\nabla f_m\left(\boldsymbol{v}^{(\tau-1)}\right)\Big]\Big\rVert\\ 
    \le&\left\lVert {
    \boldsymbol{u}_n^{(\tau-1)}-\boldsymbol{v}^{(\tau-1)}}\right\rVert+\eta\Delta_n\\
    &+\eta \sum\limits_{m\in\mathcal{E}_n}\alpha_{m,n}
    \left\lVert {
    \nabla f_m\left(\boldsymbol{w}_m^{(\tau-1)}\right)-\nabla f_m\left(\boldsymbol{v}^{(\tau-1)}\right)}\right\rVert
    \\
    =&\left\lVert {\boldsymbol{u}_n^{(\tau-1)}\!-\!\boldsymbol{v}^{(\tau-1)}} \right\rVert \!+\! \eta\beta\sum\limits_{m\in \mathcal{E}_n}\alpha_{m,n}
    \left\lVert  {\boldsymbol{w}_m^{(\tau-1)}\!-\!\boldsymbol{v}^{(\tau-1)}}\right\rVert\!+\!\eta\Delta_n.
\end{align}
\end{proof}

\begin{lemma}
    Denote $\tau=k\tau_l\tau_e+\tau_0, \tau_0\in(0,\tau_l\tau_e]$, and we have the following for local parameters:

\begin{align}
    &\left\lVert {\boldsymbol{w}_m^{(\tau)}-\boldsymbol{\Tilde{v}}^{(\tau)}} \right\rVert \le
    \frac{\delta_m}{\beta}[\left(1+\eta\beta\right)^{\tau_0}-1].
    \label{localdiff}
\end{align}
    
\end{lemma}

\begin{proof}
    \begin{align}
        &\left\lVert {\boldsymbol{w}_m^{(\tau)}\!-\!\boldsymbol{\Tilde{v}}^{(\tau)}} \right\rVert\\
        =&
        \Big\lVert \Big[\boldsymbol{w}_m^{(\tau-1)}\!-\!\eta\nabla f_m\!\left(\!\boldsymbol{w}_m^{(\tau-1)}\!\right)\!\Big]\!-\!
        \Big[\boldsymbol{v}^{(\tau-1)}\!-\!\eta\nabla f_m\!\left(\!\boldsymbol{v}^{(\tau-1)}\!\right)\!\Big]\\
        &\!-\!
        \Big[\eta\nabla f_m\!\left(\!\boldsymbol{v}^{(\tau-1)}\!\right)\!\!-\!\eta\nabla F\!\left(\!\boldsymbol{v}^{(\tau-1)}\!\right)\!\Big] \Big\rVert\\
        =&\Big\lVert \Big[\boldsymbol{w}_m^{(\tau-1)}\!-\!\boldsymbol{v}^{(\tau-1)}\Big]\!-\!
        \Big[\eta\nabla f_m\!\left(\!\boldsymbol{w}_m^{(\tau-1)}\!\right)\!\!-\!\eta\nabla f_m\!\left(\!\boldsymbol{v}^{(\tau-1)}\!\right)\!\Big]\\
        &\!-\!
        \Big[\eta\nabla f_m\!\left(\!\boldsymbol{v}^{(\tau-1)}\!\right)\!\!-\!\eta\nabla F\!\left(\!\boldsymbol{v}^{(\tau-1)}\!\right)\!\Big] \Big\rVert\\
        \le&\!\left(\!1+\eta\beta\!\right)\!\left\lVert {\boldsymbol{w}_m^{(\tau-1)}-\boldsymbol{v}^{(\tau-1)}} \right\rVert +\eta\delta_m.\label{localdiff_1}
    \end{align}
    
    We then prove \eqref{localdiff} by mathematical induction.
    
    Firstly, the conclusion holds when $\tau_0=1$, since $ {\boldsymbol{w}_m^{(\tau-1)}=\boldsymbol{v}^{(\tau-1)}}$ when $\tau_l\tau_e\mid\tau-1$.
    Secondly, if the conclusion holds for $\tau_0=t-1$, then we have the following:
    
    \begin{align}
        &\left\lVert {\boldsymbol{w}_m^{\left(k\tau_l\tau_e+t\right)}-\boldsymbol{\Tilde{v}}^{\left(k\tau_l\tau_e+t\right)}} \right\rVert
        \\
        \le&\left(1+\eta\beta\right)\left\lVert {\boldsymbol{w}_m^{\left(k\tau_l\tau_e+t-1\right)}-\boldsymbol{v}^{\left(k\tau_l\tau_e+t-1\right)}} \right\rVert +\eta\delta_m\\
        =&\left(1+\eta\beta\right)\left\lVert {\boldsymbol{w}_m^{\left(k\tau_l\tau_e+t-1\right)}-\boldsymbol{\Tilde{v}}^{\left(k\tau_l\tau_e+t-1\right)}} \right\rVert +\eta\delta_m\\
        =&\left(1+\eta\beta\right)\frac{\delta_m}{\beta}[\left(1+\eta\beta\right)^{t-1}-1] +\eta\delta_m\\
        =&\frac{\delta_m}{\beta}[\left(1+\eta\beta\right)^{t}-1].
    \end{align}
    where the inequality comes from \eqref{localdiff_1}, and the first equality holds because
    \begin{align}
        \boldsymbol{v}^{(\tau)}=\boldsymbol{\Tilde{v}}^{(\tau)}, \quad \tau_l\tau_e\nmid\tau. \label{vtv}
    \end{align}
    
    So the conclusion holds for $\tau_0=t$, and the proposition is proved.
\end{proof}

\begin{lemma}
    Denote $\tau=k\tau_l\tau_e+\tau_0, \tau_0\in(0,\tau_l\tau_e]$, and we have the following for virtual edge parameter:
    \begin{align}
        \left\lVert {\boldsymbol{u}_n^{(\tau)}-\boldsymbol{\Tilde{v}}^{(\tau)}} \right\rVert \le& 
        \frac{\delta_n}{\beta}[\left(1+\eta\beta\right)^{\tau_0}-1]
        \\&-\eta\left(\delta_n-\Delta_n\right)[\tau_0+h\left(\tau_0\right)],
        \label{edgediff}
    \end{align}
where $h\left(t\right)=\tau_l\big[\sum\limits_{r=1}^{R_t} \left(1+\eta\beta\right)^r-R_t\big]$ and $R_t=\lfloor{\frac{t-1}{\tau_l}}\rfloor$.
\end{lemma}

\begin{proof}
    We prove it by mathematical induction.
    
    Firstly, when $\tau_0=1$, from \eqref{unv} and $\boldsymbol{u}_n^{\left(k\tau_l\tau_e\right)}=\boldsymbol{\Tilde{v}}^{\left(k\tau_l\tau_e\right)}$, we obtain 
    \begin{equation}
        \left\lVert {\boldsymbol{u}_n^{(\tau)}-\boldsymbol{\Tilde{v}}^{(\tau)}} \right\rVert \le \eta\Delta_n.
    \end{equation}
    Take $\tau_0=1$ into \eqref{edgediff} and we get the same results.
    
    Secondly, if the conclusion holds for $\tau_0=t-1$, and $\tau_l\nmid t-1$, then we have
    
    \begin{align}
        &\left\lVert {\boldsymbol{u}_n^{\left(k\tau_l\tau_e+t\right)}\!-\!\Tilde{\boldsymbol{v}}^{\left(k\tau_l\tau_e+t\right)}} \right\rVert \\
        \le&\left\lVert {\boldsymbol{u}_n^{\left(k\tau_l\tau_e+t-1\right)}\!-\!\boldsymbol{v}^{\left(k\tau_l\tau_e+t-1\right)}} \right\rVert \\ 
         &\!+\!\eta\beta\!\sum\limits_{m\in \mathcal{E}_n}\alpha_{m,n}
        \left\lVert  {\boldsymbol{w}_m^{\left(k\tau_l\tau_e+t-1\right)}\!-\!\boldsymbol{v}^{\left(k\tau_l\tau_e+t-1\right)}}\right\rVert\!+\!\eta\Delta_n\\
        =&\frac{\delta_n}{\beta}[\left(1\!+\!\eta\beta\right)^{t-1}\!-\!1]\!-\!\eta\left(\delta_n\!-\!\Delta_n\right)[\left(t\!-\!1\right)\!+\!h\left(t\!-\!1\right)]\\
        &\!+\!\eta\beta\!\sum\limits_{m\in \mathcal{E}_n}\alpha_{m,n}
        \frac{\delta_m}{\beta}[\left(1\!+\!\eta\beta\right)^{t-1}\!-\!1]\!+\!\eta\Delta_n\\
        =&\left(1\!+\!\eta\beta\right)\frac{\delta_n}{\beta}[\left(1\!+\!\eta\beta\right)^{t-1}\!-\!1]\!+\!\eta\Delta_n\\
        &\!-\!\eta\left(\delta_n\!-\!\Delta_n\right)[\left(t\!-\!1\right)+h\left(t\!-\!1\right)]\\
        =&\frac{\delta_n}{\beta}[\left(1\!+\!\eta\beta\right)^{t}\!-\!1]\!-\!\eta\left(\delta_n\!-\!\Delta_n\right)[t\!+\!h\left(t\!-\!1\right)]\\
        =&\frac{\delta_n}{\beta}[\left(1\!+\!\eta\beta\right)^{t}\!-\!1]\!-\!\eta\left(\delta_n\!-\!\Delta_n\right)[t\!+\!h\left(t\right)],
    \end{align}
    where the first inequality comes from \eqref{unv} and \eqref{vtv}, the first equality comes from induction and \eqref{localdiff}, and the last equality holds because $R_t=R_{t-1}$ when $\tau_l\nmid t-1$.
    
    Thirdly, if the conclusion holds for $\tau_0=t-1$, and $\tau_l\mid t-1$, then we have
    \begin{align}
        &\left\lVert {\boldsymbol{u}_n^{\left(k\tau_l\tau_e\!+\!t\right)}\!-\!\Tilde{\boldsymbol{v}}^{\left(k\tau_l\tau_e\!+\!t\right)}} \right\rVert \\
        \le&
        \left(1\!+\!\eta\beta\right)\left\lVert {\boldsymbol{u}_n^{\left(k\tau_l\tau_e\!+\!t\!-\!1\right)}\!-\!\boldsymbol{v}^{\left(k\tau_l\tau_e\!+\!t\!-\!1\right)}}\right\rVert\!+\!\eta\Delta_n\\
        =&\left(1\!+\!\eta\beta\right)\Big{\{}\frac{\delta_n}{\beta}[\left(1\!+\!\eta\beta\right)^{t\!-\!1}\!-\!1]\!-\!\eta\left(\delta_n\!-\!\Delta_n\right)[\left(t\!-\!1\right)\!+\!h\left(t\!-\!1\right)]\Big{\}}\\
        &\!+\!\eta\Delta_n\\
        =&\frac{\delta_n}{\beta}[\left(1\!+\!\eta\beta\right)^{t}\!-\!1]\\
        &\!-\!\eta\left(\delta_n\!-\!\Delta_n\right)\Big[t\!+\!\left(1\!+\!\eta\beta\right)h\left(t\!-\!1\right)\!+\!\eta\beta\left(t\!-\!1\right)\Big].
        \label{edgediff_1}
    \end{align}
    Taking the definition of $h\left(t\right)$ into the last but two term, we obtain
    \begin{align}
        \left(1\!+\!\eta\beta\right)h\left(t\!-\!1\right)
        &=\left(1\!+\!\eta\beta\right)\tau_l\big[\sum\limits_{r=1}^{R_{t\!-\!1}} \left(1\!+\!\eta\beta\right)^r\!-\!R_{t\!-\!1}\big]\\
        &=\tau_l\big[\sum\limits_{r=2}^{R_{t\!-\!1}\!+\!1} \left(1\!+\!\eta\beta\right)^r\!-\!\left(1\!+\!\eta\beta\right)R_{t\!-\!1}\big]\\
        &=\tau_l\big[\sum\limits_{r=1}^{R_{t\!-\!1}\!+\!1} \left(1\!+\!\eta\beta\right)^r\!-\!R_{t\!-\!1}\!-\!1\!-\!\eta\beta\left(R_{t\!-\!1}\!+\!1\right)\big]\\
        &=\tau_l\big[\sum\limits_{r=1}^{R_t} \left(1\!+\!\eta\beta\right)^r\!-\!R_t\!-\!\eta\beta R_t\big]\\
        &=h\left(t\right)\!-\!\eta\beta\left(R_t\tau_l\right)\label{edgediff_2}.
    \end{align}
    Taking \eqref{edgediff_2} into \eqref{edgediff_1}, and using the fact that $R_t\tau_l=t-1$ when $\tau_l\mid t-1$, we finally get
    \begin{align}
        &\left\lVert {\boldsymbol{u}_n^{\left(k\tau_l\tau_e+t\right)}-\Tilde{\boldsymbol{v}}^{\left(k\tau_l\tau_e+t\right)}} \right\rVert \\
        &\le
        \frac{\delta_n}{\beta}[\left(1+\eta\beta\right)^{t}-1]-\eta\left(\delta_n-\Delta_n\right)[t+h\left(t\right)],
    \end{align}
    so the proposition is proved.
\end{proof}
Now we can turn to prove Theorem \ref{theo1}.
    By substituting \eqref{localdiff} and \eqref{edgediff} into \eqref{clouddiff} we obtain
\begin{align}
    &\left\lVert{\boldsymbol{u}^{(\tau)}\!-\!\boldsymbol{\Tilde{v}}^{(\tau)}}\right\rVert\le\\
    &\begin{cases} 
        0,\quad \tau_l\tau_e\mid\tau\!-\!1,\\
        \left\lVert{\boldsymbol{u}^{(\tau-1)}\!-\!\boldsymbol{v}^{(\tau-1)}}\right\rVert+\eta\delta [\left(1+\eta\beta\right)^{\tau\!-\!1}\!-\!1],\tau_l\nmid\tau\!-\!1, \\ 
        \left\lVert{\boldsymbol{u}^{(\tau-1)}\!-\!\boldsymbol{v}^{(\tau-1)}}\right\rVert+\eta\delta [\left(1+\eta\beta\right)^{\tau\!-\!1}\!-\!1]\\ 
        \!-\! \eta\left(\delta\!-\!\Delta\right)[(\tau\!-\!1)+h(\tau\!-\!1)],
        \tau_l\mid\tau\!-\!1,\tau_l\tau_e\nmid\tau\!-\!1. \end{cases}
\end{align}
    Summing up over $\tau_0$ we obtain
    \begin{align}
        &\left\lVert{\boldsymbol{u}^{\left(k\tau_l\tau_e+\tau_0\right)}\!-\!\boldsymbol{\Tilde{v}}^{\left(k\tau_l\tau_e+\tau_0\right)}}\right\rVert\\
        \le&
        \eta\delta\sum\limits_{t=1}^{\tau_0\!-\!1}\big[\left(1+\eta\beta\right)^t\!-\!1\big]\!-\!\sum\limits_{r=1}^{R_{t\!-\!1}}
        \eta\left(\delta\!-\!\Delta\right)\big[r\tau_l+ h\left(r\tau_l\right)\big]\\
        =&\frac{\delta}{\beta}\big[\left(1+\eta\beta\right)^{\tau_0}\!-\!1\big]\!-\!\eta\delta\tau_0\!-\!\eta\left(\delta\!-\!\Delta\right)\\
        & \Big[\frac12 R_{\tau_0\!-\!1}\left(1+R_{\tau_0\!-\!1}\right)\tau_l+\sum\limits_{r=1}^{R_{\tau_0\!-\!1}}h\left(r\tau_l\right)\Big]. \label{clouddiff_1}
    \end{align}
    Taking $\tau_0=\tau_l\tau_e$ into \eqref{clouddiff_1}, and letting $k=k\!-\!1$, we finally get
    \begin{align}
        U_k=&\left\lVert{\boldsymbol{u}^{\left(k\tau_l\tau_e\right)}\!-\!\boldsymbol{\Tilde{v}}^{\left(k\tau_l\tau_e\right)}}\right\rVert\\
        \le&
        \frac{\delta}{\beta}\big[\left(1\!+\!\eta\beta\right)^{\tau_l\tau_e}\!-\!1\big]\!-\!\tau_l\tau_e\eta\delta\\
        &\!-\!\eta\left(\delta\!-\!\Delta\right)\Big[
        \frac12\tau_e\left(\tau_e\!-\!1\right)\tau_l\!+\!H\Big],
    \end{align}
    where $H=\sum\limits_{r=1}^{\tau_e\!-\!1}h\left(r\tau_l\right)$.


\section{Proof of Theorem \ref{theo2}}\label{app3}
\begin{lemma}
    For the mobility case, we have the following for the virtual cloud parameter:
    \begin{align}
        &\left\lVert{\boldsymbol{u}^{(\tau)}\!-\!\boldsymbol{\Tilde{v}}^{(\tau)}}\right\rVert\le\\
        &\begin{cases}
        \left\lVert{\boldsymbol{u}^{(\tau-1)}\!-\!\boldsymbol{v}^{(\tau-1)}}\right\rVert\!+\!\eta\beta \sum_m\alpha_m \left\lVert{\boldsymbol{w}_m^{(\tau-1)}\!-\!\boldsymbol{v}^{(\tau-1)}}\right\rVert,\\
        \hspace{13em}\tau_l\nmid\tau\!-\!1,\\ 
        \left\lVert{\boldsymbol{u}^{(\tau-1)}\!-\!\boldsymbol{v}^{(\tau-1)}}\right\rVert\!+\!\eta\beta \sum_n\theta_n \left\lVert{\hat{\boldsymbol{u}}_{n}^{(\tau-1)}\!-\!\boldsymbol{v}^{(\tau-1)}}\right\rVert,\\
        \hspace{13em}\tau_l\mid\tau\!-\!1,\tau_l\tau_e\nmid\tau\!-\!1,\\
        0,\hspace{12.05em} \tau_l\tau_e\mid\tau\!-\!1. \end{cases}\label{mobclouddiff}
    \end{align}
\end{lemma}
We can easily prove it according to the definition.
\begin{lemma}
    Denote $\tau=k\tau_l\tau_e+\tau_0, \tau_0\in(0,\tau_l\tau_e]$, and we have the following for the virtual edge mobility parameter:
    \begin{flalign}
        &\left\lVert {\hat{\boldsymbol{u}}_{n}^{(\tau)}-\boldsymbol{\Tilde{v}}^{(\tau)}} \right\rVert \le\frac{\delta_n^{(\tau)}}{\beta}[\left(1\!+\!\eta\beta\right)^{\tau_0}\!-\!1]\!-\!\eta\tau_0\!\left(\!\delta_n^{(\tau)}\!-\!\Delta_n^{(\tau)}\!\right)\!.&\label{mobedgediff}
\end{flalign}
\end{lemma}
This is a simple extension of \eqref{edgediff} if we consider a no-mobility training procedure where vehicles follow the topology of $\tau$-th iteration.

Taking \eqref{mobedgediff} into \eqref{mobclouddiff} and summing up over $\tau_0$, we obtain
\begin{align}
        \left\lVert{\boldsymbol{u}^{(\tau)}\!-\!\boldsymbol{\Tilde{v}}^{(\tau)}}\right\rVert\le&
        \eta\delta\sum\limits_{t=1}^{\tau_0\!-\!1}\big[\left(1\!+\!\eta\beta\right)^t\!-\!1\big]\\
        &\!-\!\sum\limits_{r=1}^{R_{t\!-\!1}}
        \Big[r\tau_l\eta\left(\delta\!-\!\Delta^{\left(k\tau_l\tau_e+r\tau_l\right)}\right)\Big].
\end{align}
Let $\tau_0=\tau_l\tau_e$ and $k=k\!-\!1$, and we obtain
\begin{align}
 U_{k,mob}=&\left\lVert{\boldsymbol{u}^{\left(k\tau_l\tau_e\right)}\!-\!\boldsymbol{\Tilde{v}}^{\left(k\tau_l\tau_e\right)}}\right\rVert \\ \le&
\frac{\delta}{\beta}[\left(1\!+\!\eta\beta\right)^{\tau_l\tau_e}\!-\!1]\!-\!\tau_l\tau_e\eta\delta\\
&\!-\!\eta\tau_l\Big[\frac12\tau_e\left(\tau_e\!-\!1\right)\delta\!-\!\sum\limits_{j=1}^{\tau_e\!-\!1}j\Delta^{[k\tau_e\!+\!j]}\Big].
\end{align}

\section{proof of Proposition \ref{prop2}}\label{app4}

    To prove the proposition, we first introduce a lemma.
    \begin{lemma}
        \cite{rosenthal1995convergence}Suppose transition probability matrix $\boldsymbol{Q}$ satisfies $\lVert\lambda_*\rVert<1$ and is diagonalizable. Then there is a unique stationary distribution $\pi$ on $X$ and, given an initial distribution $\mu_0$ and point $x\in X$, there is a constant $l_x>0$ such that
        \begin{equation}
            \lVert\mu_k\left(x\right)-\pi\left(x\right)\rVert_1\le l_x\lVert\lambda_*\rVert^{j}.\label{convspeed}
        \end{equation}\label{mcconv}\vspace{-15pt}
    \end{lemma}
    
    Based on the lemma, we further let $\boldsymbol{Q}^{[j]}=[\boldsymbol{s}_1^{[j]};\boldsymbol{s}_2^{[j]},...;\boldsymbol{s}_C^{[j]}]$, so the evolution of $\boldsymbol{Q}^{[j]}$ can be regarded as $C$ separate distributions that converge to steady distribution. Assume the global dataset is balanced, and thus we have $\lVert\boldsymbol{s}_c^{[j]}\rVert_1=\lVert\boldsymbol{s}_{c'}^{[j]}\rVert_1$ for all $c,c'$, and we can calculate them by

    \begin{equation}
        \sum\limits_{c=1}^C \lVert\boldsymbol{s}_c^{[j]}\rVert_1 = \sum\limits_{n=1}^N \lVert\boldsymbol{q}_n^{[j]}\rVert_1 = N,
    \end{equation}
    so $\lVert\boldsymbol{s}_c^{[j]}\rVert_1=\frac{N}{C}$ for all $c$. Then we define $\boldsymbol{\Tilde{s}}_c^{[j]}=\frac{C}{N}\boldsymbol{s}_c^{[j]}$ to make it regularized distribution. 
    
    Taking the distribution above into \eqref{mcconv}, we get
    \begin{equation}
    \lVert\boldsymbol{\Tilde{s}}_c^{[j]}\left(n\right)-\frac1N\rVert_1\le l_{n,c}\lVert\lambda_*\rVert^{j},
    \end{equation}
    which holds for all $n$.
    
    So we can calculate the probability divergence by
    \begin{align}
        \lVert\boldsymbol{q}-\boldsymbol{q}^{[j]}_{n}\rVert_1
        &=\sum\limits_{c=1}^C\lVert\boldsymbol{s}_c^{[j]}\left(n\right)-\frac1C\rVert_1=\sum\limits_{c=1}^C\lVert\frac{N}{C}\boldsymbol{\Tilde{s}}_c^{[j]}\left(n\right)-\frac1C\rVert_1\hspace{-5mm}\\
        &=\frac{N}{C}\sum\limits_{i=1}^C\lVert\boldsymbol{\Tilde{s}}_c^{[j]}\left(n\right)-\frac1N\rVert_1 \le NL_n\lVert\lambda_*\rVert^{j},
    \end{align}
    where $L_n=\sum\limits_{c=1}^C l_{n,c}$.



 




\vfill


\end{document}